\documentclass[accepted]{uai2026} % after acceptance, for a revised version; 
\usepackage[american]{babel}
\usepackage{natbib} % has a nice set of citation styles and commands
\usepackage{mathtools} % amsmath with fixes and additions
\usepackage{booktabs} % commands to create good-looking tables
\usepackage{tikz} % nice language for creating drawings and diagrams

\usepackage{amsmath}
\usepackage{amssymb}
\usepackage{amsthm}
\usepackage{amsmath,amsfonts,bm}

\def\eqref#1{equation~\ref{#1}}
\def\1{\bm{1}}

\def\vpi{{\bm{\pi}}}

\def\vb{{\bm{b}}}

\def\ve{{\bm{e}}}

\def\vh{{\bm{h}}}

\def\vp{{\bm{p}}}
\def\vq{{\bm{q}}}

\def\vx{{\bm{x}}}

\DeclareMathAlphabet{\mathsfit}{\encodingdefault}{\sfdefault}{m}{sl}
\SetMathAlphabet{\mathsfit}{bold}{\encodingdefault}{\sfdefault}{bx}{n}

\def\gA{{\mathcal{A}}}

\def\gC{{\mathcal{C}}}
\def\gD{{\mathcal{D}}}

\def\gL{{\mathcal{L}}}

\def\gP{{\mathcal{P}}}
\def\gQ{{\mathcal{Q}}}

\def\gS{{\mathcal{S}}}

\def\gX{{\mathcal{X}}}
\def\gY{{\mathcal{Y}}}

\def\sQ{{\mathbb{Q}}}
\def\sR{{\mathbb{R}}}

\newcommand{\E}{\mathbb{E}}

\DeclareMathOperator*{\argmin}{arg\,min}

\usepackage{dsfont}
\usepackage{multirow}
\usepackage[dvipsnames]{xcolor}
\usepackage{tabularx}
\usepackage{subcaption}
\definecolor{colorblind_green}{HTML}{59a14f}
\definecolor{colorblind_red}{HTML}{e15759}

\usepackage{pifont}% http://ctan.org/pkg/pifont
\newcommand{\Dcal}{\gD_\mathrm{cal}}

\newcommand{\quantile}[2]{\sQ\left(#1;#2\right)}
\newcommand{\convex}{\mathrm{Convex}}
\renewcommand{\Pr}{\operatorname{P}}

\usepackage[capitalize,noabbrev]{cleveref}

\theoremstyle{plain}
\usepackage{aliascnt}

\theoremstyle{plain}
\newtheorem{theorem}{Theorem}[section]

\newaliascnt{proposition}{theorem}
\newtheorem{proposition}[proposition]{Proposition}
\aliascntresetthe{proposition}
\crefname{proposition}{Proposition}{Propositions}

\newaliascnt{lemma}{theorem}

\aliascntresetthe{lemma}
\crefname{lemma}{Lemma}{Lemmas}

\newaliascnt{corollary}{theorem}

\aliascntresetthe{corollary}
\crefname{corollary}{Corollary}{Corollaries}

\theoremstyle{definition}
\newaliascnt{definition}{theorem}
\newtheorem{definition}[definition]{Definition}
\aliascntresetthe{definition}
\crefname{definition}{Definition}{Definitions}

\newaliascnt{assumption}{theorem}

\aliascntresetthe{assumption}
\crefname{assumption}{Assumption}{Assumptions}

\theoremstyle{remark}
\newaliascnt{remark}{theorem}

\aliascntresetthe{remark}
\crefname{remark}{Remark}{Remarks}

\newtheorem{desideratum}{Desideratum}
\crefname{desideratum}{Desideratum}{Desiderata}

\title{Optimal Conformal Prediction under Epistemic Uncertainty}

\makeatletter
\renewcommand\AB@affilsepx{\quad\protect\Affilfont}
\makeatother

\author[1,2]{\href{mailto:<alireza.javanmardi@ifi.lmu.de>?Subject=Your UAI 2026 paper}{Alireza Javanmardi}{}}
\author[3]{Soroush H. Zargarbashi}
\author[1,2]{Santo M. A. R. Thies}
\author[4]{Willem Waegeman}
\author[5]{Aleksandar Bojchevski}
\author[1,2,6]{Eyke H\"ullermeier}

\affil[1]{LMU Munich}
\affil[2]{MCML}
\affil[3]{CISPA}
\affil[4]{Ghent University}
\affil[5]{University of Cologne}
\affil[6]{DFKI}

\makeatletter
\AtBeginDocument{%
  \gdef\AB@authlist{%
    \protect\Authfont \href{mailto:alireza.javanmardi@ifi.lmu.de?Subject=Your UAI 2026 paper}{Alireza Javanmardi}\textsuperscript{\normalfont 1,2}%
    \Authsep\protect\Authfont Soroush H. Zargarbashi\textsuperscript{\normalfont 3}%
    \Authsep\protect\Authfont Santo M. A. R. Thies\textsuperscript{\normalfont 1,2}%
    \\\protect\Authfont Willem Waegeman\textsuperscript{\normalfont 4}%
    \Authsep\protect\Authfont Aleksandar Bojchevski\textsuperscript{\normalfont 5}%
    \Authsep\protect\Authfont Eyke H\"ullermeier\textsuperscript{\normalfont 1,2,6}%
  }%
}
\makeatother
  
\begin{document}
\maketitle
\begin{abstract}  
Conformal prediction (CP) is a widely used frequentist framework to quantify uncertainty by constructing prediction sets with user-specified marginal coverage guarantees. In practice, CP is typically applied on top of probabilistic classifiers, which are able to express aleatoric but not epistemic uncertainty. In this paper, we consider the question of how to optimally employ CP on top of a more expressive formalism, namely credal sets, which can express both aleatoric and epistemic uncertainty. More specifically, we propose probabilistic Bernoulli prediction sets (BPS) and derive a variant that achieves conditional coverage for valid credal sets while remaining minimal in expected size. We then address the more realistic scenario in which the validity of the credal sets is not guaranteed. Assuming access to calibration data with ground-truth distributions over labels, we apply conformal risk control to BPS and derive a PAC-style guarantee: with high probability over the data, the achieved conditional coverage is at least the desired level. We validate our theoretical findings empirically over various datasets.
\end{abstract}

% Weaker ideal assumption, in a non-ideal world why should we have + rosselini why we do not compare.

\section{Introduction}
Modern neural networks are widely deployed in safety-critical applications, where reliable uncertainty quantification is essential. In practice, uncertainty is typically modeled probabilistically, for instance, using a probabilistic classifier (e.g., a neural network with a softmax output), which predicts a label distribution for each input.
% for example, in classification via probabilistic classifiers that predict a probability distribution over labels for each input. 
Not only are such predictions often miscalibrated \citep{guo2017calibration}, but they also fail to distinguish between two distinct sources of uncertainty, commonly known as \emph{aleatoric} and \emph{epistemic} uncertainty \citep{hora1996aleatory, hullermeier2021aleatoric}. Aleatoric uncertainty arises from inherent randomness in the data and is irreducible, whereas epistemic uncertainty reflects the learner's lack of knowledge and can, in principle, be reduced.
More expressive representations of uncertainty that are able to capture both aleatoric and epistemic uncertainty exist in two main forms: credal sets, which are convex sets of plausible probability distributions over labels \citep{walley1991statistical, zaffalon2002naive, caprio2024credal, caprio2024credalbayesian, wang2024credal, lohr2025credal, nguyen2025credal, hofman2026efficient}, and second-order distributions, which are distributions over label distributions \citep{neal2012bayesian, blundell2015weight, daxberger2021laplace, sensoy2018evidential}. Despite their appealing properties, such representations lack a formal notion of reliability in the uncertainties they express.
% (e.g., evidential models \citep{sensoy2018evidential} or Bayesian neural networks \citep{neal2012bayesian} and their approximations via variational inference \citep{blundell2015weight}, Monte Carlo Dropout \citep{gal2016dropout}, deep ensembles \citep{lakshminarayanan2017simple}, and Laplace approximation \citep{daxberger2021laplace}). 

Formal reliability guarantees are provided by conformal prediction (CP) \citep{vovk2022algorithmic}, a framework for uncertainty representation based on set-valued predictions: rather than returning a single label, CP outputs a set of plausible labels. Such prediction sets are practical objects for risk-averse decision-making, as highlighted by \citet{pmlr-v267-kiyani25a}.
% Importantly, CP provides a notion of reliability by guaranteeing that the true label is contained in the predicted set with any user-specified probability. 
Importantly, CP guarantees marginal coverage, that is, on average, over a random draw of the data, the true label is contained in the predicted set with the user-specified probability.
This guarantee is distribution-free, holds in finite samples, and is agnostic to the underlying predictive model, requiring only a set of holdout calibration data that are exchangeable with the future test data. 
Conformal risk control (CRC) \citep{angelopoulos2022conformal} proposes similar distribution-free guarantees to control an arbitrary user-defined risk function\,---\,it ensures that risk is upper bounded for the future test point. CP is a particular case of CRC with set-valued prediction and miscoverage risk.
% ORIGINAL: Conformal risk control (CRC) \citep{angelopoulos2022conformal} extends CP by enabling distribution-free control of arbitrary user-defined risk functions. In particular, CRC ensures that the expected risk of a set predictor is upper-bounded by a prescribed error level.

In practice, CP is most commonly applied on top of probabilistic classifiers and has demonstrated strong empirical performance across a wide range of classification tasks \citep{sadinle2019least, romano2020classification, angelopoulos2020uncertainty}. The core component of CP is the nonconformity score, a function that quantifies the (im)plausibility of assigning a candidate label to a given input. When combined with probabilistic classifiers, this score is typically defined as a function of the predicted label distribution, e.g., the negative predicted probability (softmax) of the label \citep{sadinle2019least}.
Among various approaches, \citet{romano2020classification} introduced the Adaptive Prediction Sets (APS) method, which, given access to the oracle label distributions, returns the smallest possible prediction sets that achieve conditional coverage, i.e., that contain the true label with a user-specified probability for each input.

In this paper, we are interested in constructing reliable prediction sets \`a la CP when uncertainty is represented by a credal set predictor rather than by a probabilistic classifier. One can look at this problem from two different perspectives. From the point of view of epistemic uncertainty representation, the question is how to ``conformalize'' a credal set predictor so as to deliver reliable set predictions. From the point of view of conformal prediction, the question is how to benefit from the additional information about epistemic uncertainty provided by a credal set predictor compared to a probabilistic classifier (see \cref{sec:related_work} for a discussion of related approaches at this intersection). 

As one important benefit, let us highlight guaranteed conditional coverage. Imagine, for example, a credal set predictor is able to guarantee that the conditional probability of label $A$ given query instance $\vx$ is between 0.6 and 0.7, the probability of label $B$ is between 0.3 and 0.4, and the probability of $C$ is between 0 and 0.1. The learner can then be sure that the prediction set $ \{A, B \}$ has a conditional coverage of (at least) 0.9, i.e., that $\{A, B \}$ is a valid prediction if the (user-specified) error rate is 0.1. This guarantee could not be given on the basis of any first-order representation, e.g., a representative label distribution of the credal set. 

We approach the problem step by step, progressing from simpler to more difficult variants.
We begin with the setting of valid credal sets, where the true label distribution is guaranteed to lie within the credal set for each input (like in the previous example). In this setting, we introduce \underline{B}ernoulli \underline{P}rediction \underline{S}ets (BPS), which take a credal set as input and return the smallest prediction set satisfying conditional coverage at any given desired level. As the assumption of valid credal predictions is unrealistic in practice, we first consider the problem of optimal set prediction under a suitably relaxed assumption of partially valid credal sets, and finally, in a scenario where no validity properties can be guaranteed at all. In these latter two cases, we obtain a probably approximately correct (PAC)-style conditional coverage guarantee: with high probability, the achieved conditional coverage is at least the desired level.

We evaluate our proposed methods on multiple datasets under settings with valid, partially valid, and unknown-validity credal sets, demonstrating the practical relevance of our theoretical results. 
% \section{Related Work}\label{sec:related_work}
% The connection between conformal prediction and epistemic uncertainty-aware predictors has recently attracted increasing attention. 
% \citet{rossellini2024integrating} incorporate epistemic uncertainty into conformalized quantile regression to improve conditional coverage. 
% \citet{karimi2024evidential} introduce a nonconformity score based on uncertainty estimates derived from evidential models. 
% \citet{cabezas2025epistemic} integrate epistemic uncertainty into conformal prediction by fitting a Bayesian model on top of nonconformity scores. 
% \citet{azizi2026clear} construct prediction sets by adaptively balancing epistemic and aleatoric uncertainty. 
% What separates our work from these approaches is that we study how epistemic uncertainty represented explicitly via credal sets can be incorporated into conformal prediction, with a focus on providing finite-sample conditional coverage guarantees in this setting. 
% The problem of deriving prediction sets from credal sets has also been studied in the imprecise probability literature \citep{caprio2024credalbayesian}. 
% However, such approaches typically do not provide finite-sample coverage guarantees without additional structural assumptions.
\section{Background}
\label{sec: Background}
We assume that inputs $ \vx_i $ are sampled i.i.d.\ from a distribution $ \gP_{\gX} $. Conditional on $ \vx_i $, the label $ y_i $ is drawn from the categorical distribution $ p(\cdot \mid \vx_i) \in \triangle^{K} $, where $ K $ is the number of classes and $ \triangle^{K} $ is the $(K-1)$-dimensional probability simplex. In this setting, we write $ \vp_i := p(\cdot \mid \vx_i) $ and refer to it as the oracle label distribution at $ \vx_i $ (which is uniquely determined by $ \vx_i $). Even if the true distribution $ \vp_i $ were known, the prediction of the label $ y_i $ is subject to aleatoric uncertainty due to the inherent randomness in sampling $ y_i \sim \vp_i $.
For a given input $ \vx_i $, a probabilistic classifier outputs a single probability distribution over labels, denoted by $ \vpi_i = \pi(\cdot \mid \vx_i) \in \triangle^K $, which serves as an estimate of $ \vp_i $. While such an estimate can represent aleatoric uncertainty, it cannot capture the discrepancy between $ \vpi_i $ and $ \vp_i $, which corresponds to epistemic uncertainty. A credal set predictor captures this discrepancy by outputting a convex set of plausible probability distributions over labels, denoted by
$ \gQ_i := \gQ(\vx_i) \subseteq \triangle^K $.
Roughly speaking, the larger the credal set, the greater the learner’s uncertainty about the true distribution $ \vp_i $.

In the standard classification setting, one typically observes \textit{zero-order} data consisting of pairs $(\vx_i, y_i)$, where each input $ \vx_i $ is paired with a single realized label $ y_i $. In a richer setting, one may instead have access to \textit{first-order} data consisting of pairs $(\vx_i, \vp_i)$, where $ \vx_i $ is paired with the true label distribution $ \vp_i $. Such first-order data can arise in practice, for example, by aggregating multiple annotations per instance into a label distribution \citep{peterson2019human, nie2020what, obuchowicz2020qualityMRI, schmarje2022benchmark}, and are therefore becoming relevant in a growing range of applications, including conformal prediction \citep{stutz2023conformal,javanmardi2024conformalized, caprio2025conformalized}.

\paragraph{Conformal prediction and risk control.}
Let $\Dcal^\mathrm{zero} = \{(\vx_i, y_i)\}_{i = 1}^n$ be a holdout zero-order calibration set exchangeable with the future test point $(\vx_{n+1}, y_{n+1})$. For any nonconformity score function $s:\gX \times \gY \to \sR$ (capturing the disagreement between $\vx$ and $y$), and any user-specified coverage rate $1 - \alpha$, conformal prediction \citep{vovk2022algorithmic} constructs sets $\gC(\vx_{n+1})$ as follows. Let $\quantile{1 - \alpha}{\gA}$ denote the $(1 - \alpha)\cdot(1 + \frac{1}{n})$ quantile of the set $\gA$ and let $q := \quantile{1 - \alpha}{\{s(\vx_i, y_i): (\vx_i, y_i) \in \Dcal^\mathrm{zero}\}}$. Then, 
\begin{align}
\label{eq:conformal:set}
\gC(\vx_{n+1}) = \{ y: s(\vx_{n+1},y) \leq q\} \, .
\end{align}
In the common classification setting, the nonconformity score is often defined in terms of the predicted label distribution of a probabilistic classifier; for example, $s(\vx, y) = -\pi(y \mid \vx)$ \citep{sadinle2019least}.
CP guarantees
\begin{align}\label{eq:conformal:marginal}
\Pr_{\gD_+}[y_{n+1} \in \gC(\vx_{n+1})] \ge 1 - \alpha,
\end{align}
with $ \gD_+ := \Dcal^\mathrm{zero} \cup \{(\vx_{n+1}, y_{n+1}) \}$. Moreover, when the scores have a continuous joint distribution, and there are no ties, the coverage probability is upper bounded by $1 - \alpha + \frac{1}{n+1}$ \citep{lei2018distribution}. This guarantee holds only marginally, i.e., on average over a random draw of calibration and test data. In contrast, conditional coverage requires the same coverage probability at every test point, i.e.,
\begin{align}\label{eq:conformal:conditional}
\Pr \big[ y_{n+1} \in \gC(\vx_{n+1}) \mid \vx_{n+1} \big] \ge 1 - \alpha.
\end{align}
Generalizing CP, one can provide a similar guarantee for arbitrary risk functions through conformal risk control (CRC) \citep{angelopoulos2022conformal}. Consider a family of set predictors ${\gC_\lambda}$ indexed by a parameter $\lambda$, and let $\gL_i(\lambda) := \gL(\gC_\lambda(\vx_i), y_i)$ denote the risk induced by the set $\gC_\lambda(\vx_i)$. For exchangeable $\gD_+$, if $\gL(\lambda)$ is non-increasing in $\lambda$, bounded above (i.e., $\gL(\lambda) \in (-\infty, b]$), and right continuous in $\lambda$, \citet{angelopoulos2022conformal} show:
\begin{align}\label{eq:conformal-risk-control:guarantee}
\E_{\gD_+}[\gL_{n+1}(\lambda^\star)] \le \alpha  
\end{align}
for 
$$
\quad \lambda^\star := \inf\Big\{\lambda: \frac{1}{n+1} \sum_{i = 1}^{n}\gL_i(\lambda) + \frac{b}{n+1} \le \alpha \Big\} \, .
$$ 
Intuitively, $\lambda$ is a conservativeness parameter: increasing $\lambda$ typically produces larger (more conservative) prediction sets and therefore reduces the risk. CRC selects the smallest $\lambda$ that ensures the expected risk remains below the predefined tolerance level $\alpha$.
Note that conformal prediction itself is a special case of CRC with $\gC_\lambda(\vx_i) = \{y: s(\vx_i, y)\leq\lambda \}$, where the risk is defined as the miscoverage; that is, $\gL_i(\lambda) = \mathds{1}[y_i \notin \gC_\lambda(\vx_i)] = \mathds{1}[s(\vx_i, y_i) > \lambda]$.

\paragraph{Adaptive prediction sets.}
CP provides a marginal guarantee, meaning it is possible to have lower coverage in some regions of $\gX$ and higher coverage in others. For instance, with $s(\vx, y) = -\pi(y \mid \vx)$, the coverage probability can be biased toward easy examples (i.e., cases with highly concentrated predicted distributions). To address this, Adaptive Prediction Sets (APS) \citep{romano2020classification} use a score function that aims to better approximate conditional coverage. Formally, the APS score is defined as $s(\vx, y) := \rho(\vx, y) + u \cdot \pi(y \mid \vx)$, where $\rho(\vx, y) := \sum_{y' \in \mathcal{Y}} \pi(y' \mid \vx) \mathds{1}[\pi(y' \mid \vx) > \pi(y \mid \vx)]$ is the cumulative probability of classes ranked above $y$, and $u \sim \mathrm{Uniform}[0,1]$ is a tie-breaking random variable used to achieve exact $1-\alpha$ coverage.\footnote{Throughout the paper, ``$\cdot$'' between two vectors denotes the inner product; between scalars, it denotes ordinary multiplication.} 
% The APS score is equivalent to including labels in decreasing order of probability until the threshold $\lambda$ is reached, i.e., $\gC_\lambda(\vx_i) = \{y: s(\vx_i, y) < \lambda\}$. Through conformal calibration, $\lambda$ is chosen so that the resulting sets achieve marginal coverage $1-\alpha$. Hence, 
Given access to the oracle label distributions, APS produces the smallest prediction sets satisfying conditional coverage \citep{angelopoulos2024theoretical}. Notably, even without such oracle access, APS tends to distribute coverage more evenly across test inputs than other CP baselines.

\section{Bernoulli prediction sets (BPS)}\label{sec:BPS}
While in CP, prediction sets are typically constructed by thresholding a score function (see (\ref{eq:conformal:set})), we develop our approach within a more general framework of randomized prediction sets. In this section, we formalize randomized, parameterized prediction sets in full generality, and in \cref{sec:Set Prediction for Credal Sets} we focus on how to derive such prediction sets from a credal set predictor in an optimal manner.

For each input $\vx_i$, a Bernoulli prediction set is specified by an \emph{inclusion probability vector} $\vb_i \in [0,1]^K$, where $b_{ik}$ denotes the probability that class $k$ is included in the predicted set. The realized (random) prediction set is obtained by sampling independent Bernoulli variables
\begin{align}
\label{eq:bps:sampling}
    z_{ik} \sim \mathrm{Bernoulli}(b_{ik}), \quad \gC_\text{BPS}(\vx_i) := \{k \in [K] : z_{ik} = 1\}.
\end{align}
Standard (deterministic) prediction sets are a special case of BPS with $\vb_i \in \{0,1\}^K$. With this probabilistic formulation, we redefine both the set size and the conditional coverage in expectation. For input $\vx_i$, the (expected) set size is
\begin{align}
\label{eq:bps:set-size}
    \E\big[|\gC_\text{BPS}(\vx_i)|\big] = \sum_{k=1}^K \Pr(z_{ik}=1) = \sum_{k=1}^K b_{ik} = \vb_i \cdot \boldsymbol{1}, 
\end{align}
and the conditional coverage is
\begin{align}
\label{eq:bps:conditional-coverage}
    \Pr\!\big(y_i \in \gC_\text{BPS}(\vx_i) \mid \vx_i\big)
    = \sum_{k=1}^K p_{ik}\Pr(z_{ik}=1)
    = \vb_i \cdot \vp_i.
\end{align}
Note that in both measures $\vb_i$ is assumed to be fixed; i.e., derived deterministically as a function of $\vx_i$. Therefore, the probability is over the randomness of the sets being sampled from $\vb_i$, and the label $y$ being sampled from $\vp_i$. In some cases, the $\vb_i$ itself will be the outcome of a process that is random (e.g., over the randomness of the calibration set in case some conformalization is used), and in that case, the metrics will be evaluated in expectation over $\vb_i$ as well.

\paragraph{Reformulating APS as BPS.} 
% To illustrate that classical set predictors can be expressed within the BPS framework, we show how 
Just to illustrate the generality of the definition, we show that adaptive prediction sets (APS), for any value of the conformal threshold, are a special case of BPS. 
% Specially APS is a instance of BPS with the vector $\vb_i$ taking its values from $\{0, b, 1\}$ for some $b$ per each $\vx_i$. 
Specifically, given a label distribution $\vpi_i$, let $\sigma_i$ be a permutation of $[K] := \{1, \ldots, K\}$ that orders the labels in decreasing order of probability, i.e.,
\begin{align}
\tilde{\pi}_{ij} := \pi_{i,\sigma_i(j)}, 
\quad \text{with} \quad
\tilde{\pi}_{i1} \ge \tilde{\pi}_{i2} \ge \dots \ge \tilde{\pi}_{iK}.
\end{align}
% Here, $j$ denotes the rank position, while $\sigma_i(j)$ refers to the corresponding original class index.

For any threshold $\tau \in [0,1]$, APS defines
$$
L(\vpi_i, \tau)
=
\min \Big\{ k \in [K] : \sum_{j=1}^{k} \tilde{\pi}_{ij} \ge \tau \Big\}
$$
and constructs a randomized set as
\begin{align}\label{eq:APS}
    \gC_\text{APS}(\vx_i,  \vpi_i, \tau, u)  = \begin{cases}
\text{top } L(\vpi_i, \tau) - 1 \text{ labels}, & \text{if } u \leq u_0 \\
\text{top } L(\vpi_i, \tau) \text{ labels}, & \text{otherwise}
\end{cases} \, 
\end{align}  
with 
\begin{align*}
u  \sim \mathrm{Uniform}[0,1] \, \text{ and }
u_0  = \tilde{\pi}_{iL(\vpi_i, \tau)}^{-1} \; \Big[ \sum_{j=1}^{L(\vpi_i, \tau)} \tilde{\pi}_{ij} - \tau \Big]  \, .
\end{align*}
In the BPS formulation, this construction corresponds to an inclusion probability vector defined over the original class indices as
\begin{align}\label{eq:APS_b}
b_{i,\sigma_i(j)} =
\begin{cases}
1, & j < L(\vpi_i, \tau) \\
1 - u_0, & j = L(\vpi_i, \tau) \\
0, & j > L(\vpi_i, \tau)
\end{cases} \, .
\end{align}

Thus, APS is a Bernoulli prediction set in which all labels with rank strictly below the threshold $L(\vpi_i, \tau)$ are deterministically included, the boundary label is included with probability $1 - u_0$, and all lower-ranked labels are excluded. While APS aims for a specific objective under a specific constraint (which is the minimum set size while preserving conditional coverage over a single predictive distribution), BPS is a general framework over which we can define more general constraints and objectives, as we further discuss. 
% This shows that APS is a special case of BPS when a single label distribution is given.
\section{Set Prediction from Credal Sets}
\label{sec:Set Prediction for Credal Sets}
We aim to construct prediction sets given a credal set predictor $ \gQ_i := \gQ(\vx_i) \subseteq \triangle^K $. A common representation of credal sets is the convex hull of $m$ label distributions (the vertices of a convex polytope), i.e., $\gQ_i = \convex(\{\vpi_{i}^{(j)}\}_{j = 1}^m)$. This representation aligns well with practice, as many credal set predictors are obtained from ensembles of probabilistic classifiers \citep{wang2024credal, wang2025credalwrapper, nguyen2025credal}. While we develop our framework for this polytope representation, more general credal sets can also be handled. In particular, if a credal set is not already given as the convex hull of finitely many points, it can first be approximated by a (finite) enclosing polytope, and the vertices of this polytope can then be used within our framework.
% While we adapt our framework to this representation, other representations can be accommodated by working with their corresponding convex hull (i.e., the minimum enclosing polytope).
% Question

There is no unique way to design prediction sets, regardless of whether the underlying predictor is a probabilistic classifier or a credal set predictor. Therefore, beyond marginal coverage, which is satisfied by any CP-based set predictor, we require additional desiderata and an explicit objective to guide the construction. For instance, achieving conditional coverage with the smallest possible set size, when given access to the true label distributions, is the principle underlying the design of APS. We aim to follow a design principle analogous to that of APS but tailored to credal set predictors. To this end, we replace the assumption of access to the true label distribution with the assumption of access to \textit{valid} credal sets, defined as follows.
\begin{definition}[valid credal sets]
\label{def:valid-credal-sets}
    A credal set $\gQ_i$ as a prediction for $\vx_i$ is said to be valid if $\vp_i \in \gQ_i$, that is, if the oracle label distribution lies within the credal set.
\end{definition}
It is worth noting that assuming access to valid credal sets is a weaker assumption than having access to oracle label distributions. Indeed, the former still accounts for epistemic uncertainty and does not require precise knowledge of the ground truth. In the following, we first consider the setting in which this validity assumption holds for every input $\vx_i$ and formulate our desiderata accordingly, leading to a concrete set predictor (\cref{Sec:CaseI}). We then study the case in which this assumption holds only with high probability over random draws of data points (\cref{Sec:CaseII}), and finally, the case where the validity of the credal sets is unknown (\cref{sec:CaseIII}).
% Next, we discuss desiderata and objectives for a credal-set predictor. 
\subsection{Case~I: valid credal sets}
\label{Sec:CaseI}
We begin with the setting in which the credal set predictor outputs a valid credal set for every input, that is, $\forall \vx_i,\ \vp_i \in \gQ_i$. Given a valid credal set $\gQ_i$, our objective is to construct a prediction set $\gC(\vx_i)$ based solely on $\gQ_i$. Since the true label distribution lies in the credal set, a natural desideratum is that the resulting prediction set achieves conditional coverage of a nominal level $1 - \alpha$. However, because the exact distribution $\vp_i$ is unknown, this requirement must hold uniformly over all distributions contained in $\gQ_i$.

\begin{desideratum}[Conditional Coverage]
\label{thrm:desideratum:conditional}
For valid credal sets, the prediction sets should satisfy conditional coverage uniformly over the credal set\,---\,formally,
\begin{equation}
\label{eq:desideratum:conditional}
\forall \vp \in \gQ_{i}:\quad
\Pr_{y \sim \mathrm{Categorical}(\vp)}\big[y \in \gC(\vx_{i})\big] \ge 1 - \alpha.
\end{equation}
\end{desideratum}
As mentioned before, credal set predictors, unlike probabilistic classifiers, are able to represent epistemic uncertainty. Prediction sets are meant to reflect the predictive uncertainty of the learner. Accordingly, a second natural desideratum is that the prediction sets deflate as epistemic uncertainty decreases. Although it is not straightforward to define a total order over arbitrary credal sets in terms of the amount of epistemic uncertainty they represent, this requirement can still be formalized under a partial order (subset) relationship: if one credal set is a subset of another, it reflects less epistemic uncertainty, and the resulting prediction set should be smaller or equal in size. Formally, the following desideratum should be achieved by any prediction set that is adaptive to epistemic uncertainty. 
\begin{desideratum}[Epistemic Adaptivity]
\label{thrm:desideratum:adaptivity}
Given two valid credal sets $\gQ_i$ and $\gQ_i'$ at point $\vx_i$ such that $\vp_i \in \gQ_i \subseteq \gQ_i'$, the prediction set constructed using $\gQ_i$ ought to be smaller or equal in size to that constructed using $\gQ_i'$.
\end{desideratum}
We design a set predictor that satisfies the above desiderata while minimizing the set size as our objective. This leads to the design of the optimal Bernoulli prediction sets.
% \clearpage
\paragraph{Optimal Bernoulli Prediction Sets.}
\label{par:OptimalBPS}
Over the previously defined BPS framework (\cref{sec:BPS}), we aim to find $\vb_i$'s with the objective of finding the smallest possible (random) prediction set that satisfies both desiderata. Formally, for a given credal set $\gQ_i = \convex(\{\vpi_{i}^{(j)}\}_{j = 1}^m)$, we determine an inclusion probability vector $\vb^\lambda_i$ as a solution to the following optimization problem:
\begin{equation}\label{eq:optimal_BPS} 
    \begin{aligned} 
    \vb_i^\lambda = \argmin_{\vb \in [0,1]^K } \vb \cdot \boldsymbol{1}  \quad \text{s.t.} \quad \forall j\in [m]:\quad \vb \cdot \vpi_i^{(j)} \ge \lambda,  
    \end{aligned} 
\end{equation}
where $\lambda$ is the threshold determining the required level of conditional coverage. 
% In particular, to achieve $1-\alpha$ conditional coverage, we set $\lambda = 1-\alpha$. We refer to the set predictor obtained with $\lambda = 1-\alpha$ as BPS$(1-\alpha)$.
In particular, to achieve $1-\alpha$ conditional coverage, we set $\lambda = 1-\alpha$ and refer to the resulting predictor as BPS$(1-\alpha)$.

It turns out that the optimization problem (\ref{eq:optimal_BPS}) is a linear program and can therefore be solved efficiently using standard solvers. In fact, it can be viewed as a multi-dimensional fractional knapsack problem.
As a consequence of this linear structure, the solution $\vb_i^\lambda$ contains at most $\min(K,m)$ fractional components, i.e., entries $b^\lambda_{ik} \in (0,1)$. 
% Stepping out of the BPS framework, even under a single constraint (which is equivalent to APS), as soon as we force $b_{ij}\in\{0, 1\}$, the problem becomes NP hard and grows exponentially with the number of classes. 
% This follows from the fact that the optimization problem in \cref{eq:optimal_BPS} is a linear program with at most $m$ coverage constraints over $K$ variables.

In the following, we show that the resulting prediction set $\gC_{\mathrm{BPS}}(\vx_i)$ is the smallest set (in expectation) that conforms with the predefined desiderata when given access to a valid credal set.
% The solution $\vb_i^\lambda$ yields the prediction set $\gC_\text{BPS}(\vx_i, \vb_i)$ with the smallest expected set size that satisfies the desired conditional coverage with respect to every distribution $\vp \in \convex(\vpi_{i}^{(1)}, \ldots, \vpi_{i}^{(m)})$.
\begin{proposition}\label{prop:optimal}
For any $\alpha \in (0,1)$ and a given credal set 
$\gQ_i = \convex(\{\vpi_{i}^{(j)}\}_{j = 1}^m)$, 
the prediction set $\gC_{\mathrm{BPS}}(\vx_i)$ in (\ref{eq:bps:sampling}), 
constructed using $\vb_i^{1-\alpha}$ from (\ref{eq:optimal_BPS}), 
is the smallest (possibly randomized) prediction set that satisfies 
$(1-\alpha)$ conditional coverage with respect to every distribution 
$\vp \in \gQ_i$.
\end{proposition}
All proofs are deferred to \cref{sec:appendix:proofs}. Moreover, 
\cref{thrm:desideratum:adaptivity} is satisfied by construction: the smallest 
prediction set that guarantees conditional coverage for a given credal set 
does not expand as epistemic uncertainty decreases. Indeed, a smaller 
credal set induces fewer constraints in (\ref{eq:optimal_BPS}), and therefore, the optimal (expected) set size cannot increase (and may decrease).

\begin{proposition}\label{prop:EUadaptive}
Consider two valid credal sets $\gQ_i$ and $\gQ_i'$ at point $\vx_i$ such that 
$\vp_i \in \gQ_i \subseteq \gQ_i'$. 
Let $\vb_i^\lambda$ and $\vb_i'^\lambda$ be the solutions to 
(\ref{eq:optimal_BPS}) corresponding to $\gQ_i$ and $\gQ_i'$, respectively. 
Then we have $\vb_i^\lambda \cdot \boldsymbol{1} 
\;\le\; 
\vb_i'^\lambda \cdot \boldsymbol{1}.$
\end{proposition}
In \cref{sec:appendix:toy}, we provide toy examples that illustrate this property in a clear and intuitive manner. Another desirable property of the set predictor in \cref{eq:optimal_BPS} is that it recovers APS when the credal set is a singleton. 
% In this case, the optimal Bernoulli inclusion vector aligns exactly with the inclusion
% probability vector derived for APS (cf. \cref{eq:APS_b}).

\begin{proposition}\label{prop:equivalence}
For a fixed threshold $\lambda$ and a singleton credal set $\gQ_i = \{\vpi_i\}$, let $\vb_i^\lambda$ be the solution of (\ref{eq:optimal_BPS}) with the single constraint given by $\vpi_i$. Then $\vb_i^\lambda$ is equal to the Bernoulli inclusion vector of APS defined in (\ref{eq:APS_b}) for the same $\vpi_i$ and the threshold $\lambda$. Consequently, the resulting prediction sets constructed by BPS and APS are equivalent in expectation.
\end{proposition}

\begin{table*}[t!]
    \centering
    \caption{Methods and guarantees across the three credal set settings.}
    \label{tab:methods}
    \resizebox{\textwidth}{!}{%
    \begin{tabular}{
        >{\centering\arraybackslash}m{2cm}
        >{\centering\arraybackslash}m{6cm}
        >{\centering\arraybackslash}m{6cm}
        >{\centering\arraybackslash}m{6cm}
    }
        \toprule
        \bfseries 
        & \textbf{Valid Credal Sets}
        & \textbf{Partially Valid Credal Sets}
        & \textbf{Credal Sets with Unknown Validity} \\
        \midrule

        &
        \includegraphics[width=0.17\textwidth]{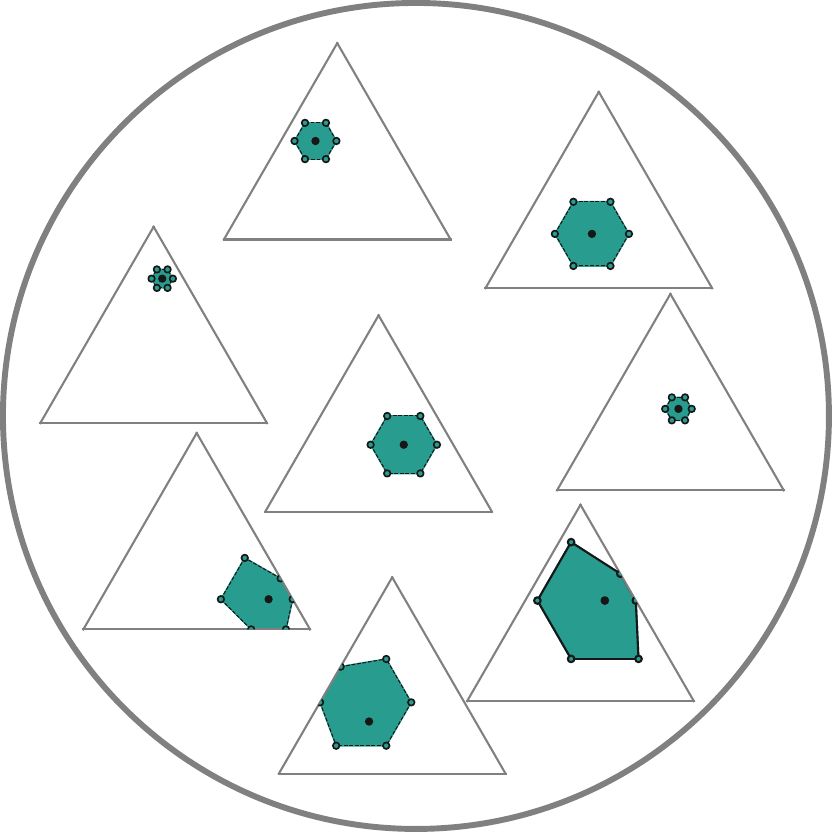}
        &
        \includegraphics[width=0.17\textwidth]{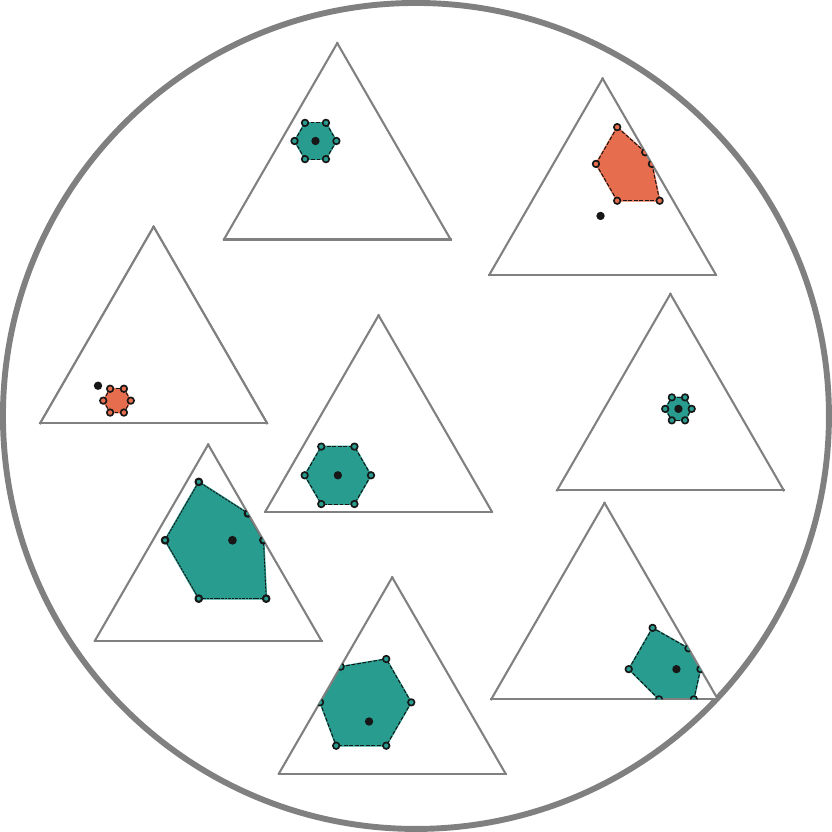}
        &
        \includegraphics[width=0.17\textwidth]{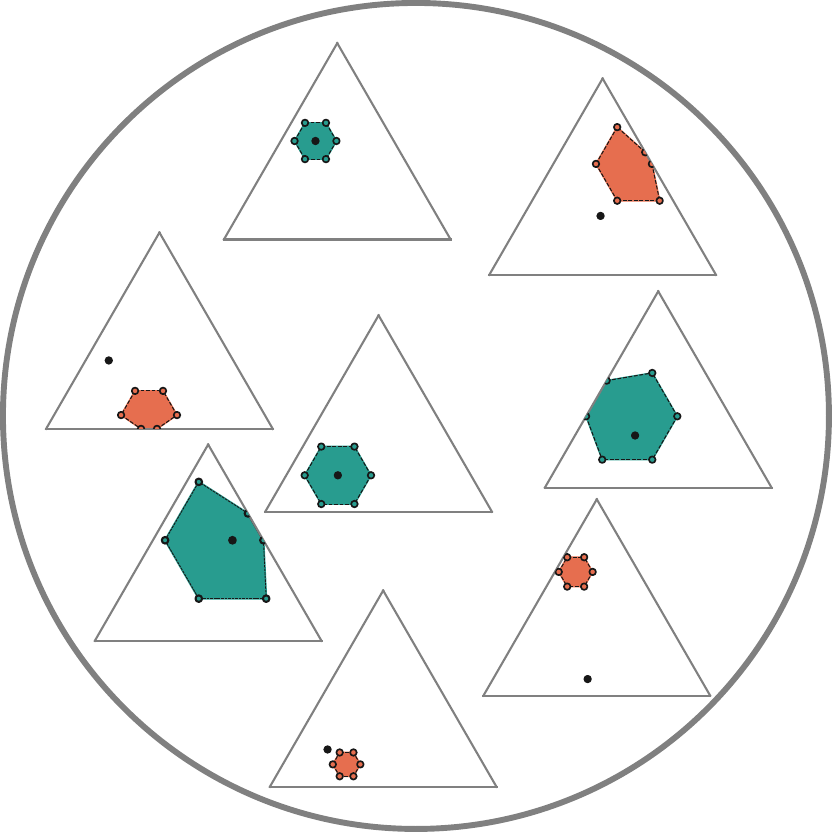}
        \\

        \bfseries
        &
        $\vp_i \in \gQ_i \;\; \forall\, \vx_i$
        &
        $\Pr_{\vx_i}\!\left[\vp_i \in \gQ_i \right] \ge 1-\epsilon$
        &
        unknown validity \\[0.8em]

        \bfseries Set Predictor
        &
        BPS$(1-\alpha)$
        &
        BPS$(1-\alpha)$
        &
        Calibrated BPS$(\lambda^\star)$
        \\[0.8em]

        \bfseries Guarantee Type
        &
        \includegraphics[width=0.2\textwidth]{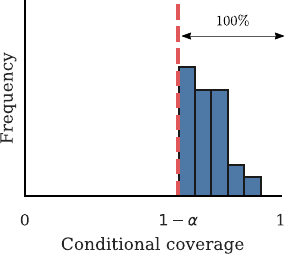}
        &
        \includegraphics[width=0.2\textwidth]{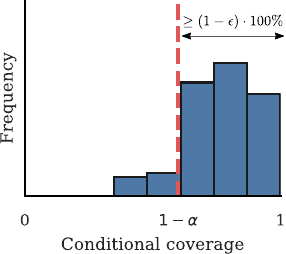}
        &
        \includegraphics[width=0.2\textwidth]{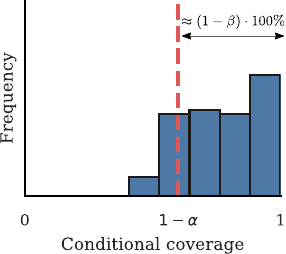}
        \\
         &
        $\vb_i^{1-\alpha} \cdot \vp_i \ge 1-\alpha \;\; \forall\, \vx_i$
        &
        $\Pr_{\vx_i}\!\left[\vb_i^{1-\alpha} \cdot \vp_i \ge 1-\alpha\right] \ge 1-\epsilon$
        &
        $\Pr_{\gD_\dagger}\!\left[\vb_{n+1}^{\lambda^\star} \cdot \vp_{n+1} \ge 1-\alpha\right] \ge 1-\beta$
        \\

        \bottomrule
    \end{tabular}
    }
\end{table*}
\subsection{Case~II: Partially valid credal sets}
\label{Sec:CaseII}
We now consider the case where credal sets are not guaranteed to be valid for every input, but instead are valid with high probability\,---\,formally,
\begin{align}\label{eq:partially_valid_credal}
    \Pr_{\vx_i}\!\left[\vp_i \in \gQ_i \right] \ge 1-\epsilon \, .
\end{align}
From a frequentist perspective, this means that over a random draw of inputs, the associated credal sets are valid for at least a $1-\epsilon$ fraction of cases.
Such partially valid credal sets may arise, for instance, from conformalized credal set predictors \citep{javanmardi2024conformalized}, where the credal sets satisfy a marginal CP guarantee of containing the label distribution with high probability.

In this setting, we again apply BPS$(1-\alpha)$ directly to $\gQ_i$, without using additional CP calibration. As a consequence, the resulting prediction sets satisfy the following PAC-style conditional coverage guarantee:
\begin{align}\label{eq:partially_valid_credal:guarantee}
    \Pr_{\vx_i}\!\left[\vb_i^{1-\alpha}\cdot \vp_i \ge 1-\alpha \right] \ge 1-\epsilon.
\end{align}

Intuitively, on at least a $(1-\epsilon)$ fraction of inputs, namely, those for which the credal sets are valid, the conditional coverage is at least at the desired $(1-\alpha)$ level.
\begin{proposition}
\label{prop:pac:partially-valid}
Suppose a credal set predictor outputs sets
$\gQ_i = \convex(\{\vpi_i^{(j)}\}_{j=1}^m)$
satisfying
\begin{align*}
\Pr_{\vx_i}\left[\vp_i \in \gQ_i\right] \ge 1-\epsilon .
\end{align*}
For any $\alpha \in (0,1)$, let $\gC_{\mathrm{BPS}}(\vx_i)$ be constructed as in (\ref{eq:bps:sampling}) using the solution $\vb_i^{1-\alpha}$ of (\ref{eq:optimal_BPS}). Then
\begin{align*}
\Pr_{\vx_i}\left[
\Pr\big(y_i \in \gC_{\mathrm{BPS}}(\vx_i) \mid \vx_i\big)
\ge 1-\alpha
\right]
\ge 1-\epsilon .
\end{align*}
\end{proposition}
It is worth noting that the validity of the credal set $\gQ_i$ is a sufficient condition for the prediction set $\vb_i^\lambda$ in (\ref{eq:optimal_BPS}) to guarantee $\lambda$ conditional coverage with respect to the true label distribution $\vp_i$, but it is not necessary. In particular, even if $\vp_i \notin \gQ_i$, the resulting Bernoulli vector may still satisfy $\vb_i^\lambda \cdot \vp_i \ge \lambda$. That being said, the guarantee in (\ref{eq:partially_valid_credal:guarantee}) should typically be interpreted as a worst-case statement. In practice, the realized conditional coverage with respect to $\vp_i$ may hold for a substantially larger fraction than $1-\epsilon$.

So far, in both Case~I (valid credal sets) and Case~II (partially valid credal sets), the prediction sets are obtained directly by solving the optimization problem in~(\ref{eq:optimal_BPS}). In these two settings, BPS$(1-\alpha)$ yields prediction sets that either satisfy conditional coverage for every input (Case~I) or satisfy it with high probability (Case~II), without requiring an additional conformal calibration step.
\subsection{Case~III: Credal Sets with Unknown Validity}\label{sec:CaseIII}
We finally consider the most general setting, in which no validity assumption is imposed on the outputs of the credal set predictor. In this case, the optimization (\ref{eq:optimal_BPS}) alone does not provide any formal guarantee on the resulting conditional coverage. Therefore, in order to obtain a distribution-free guarantee, we resort to a conformal framework. Since the predictor in (\ref{eq:optimal_BPS}) is monotone in $\lambda$ (i.e., the resulting sets inflate as $\lambda$ increases), we adopt the conformal risk control framework, which calibrates $\lambda$ to control a user-defined notion of risk.

As in the previous two cases, our goal is to obtain a conditional-type guarantee with respect to the true label distribution $\vp_i$. To this end, we consider access to a holdout first-order calibration data set $\Dcal^\mathrm{first} = \{(\vx_i, \vp_i)\}_{i=1}^n$ that is exchangeable with the future test point $(\vx_{n+1}, \vp_{n+1})$. We define the risk as the indicator that the conditional coverage falls below the target level $1-\alpha$, i.e.,
\begin{align}\label{eq:risk:satisfaction}
\gL(\vx_i, \lambda) 
:= \mathds{1}\!\left[\vb_i^\lambda \cdot \vp_i < 1-\alpha\right].
\end{align}

For the set predictor in (\ref{eq:optimal_BPS}), the calibration data $\Dcal^\mathrm{first}$, and the risk defined in (\ref{eq:risk:satisfaction}), CRC determines the threshold $\lambda^\star$ as the solution to 
\begin{equation}\label{eq:lambda_CP:satisfaction}
\inf \left\{ \lambda : 
\sum_{i=1}^n 
\mathds{1}\!\left[\vb_i^\lambda \cdot \vp_i \ge 1-\alpha\right] 
\ge \left\lceil (1-\beta)(n+1) \right\rceil 
\right\}, 
\end{equation}
which bounds the expected risk at the user-specified tolerance $\beta \in [0,1]$. This translates into the guarantee that
\begin{align}\label{eq:guarantee:unknown}
\Pr_{\gD_\dagger}\!\left[
\vb_{n+1}^{\lambda^\star} \cdot \vp_{n+1} \ge 1-\alpha
\right] \geq 1-\beta, 
\end{align}
where $\gD_\dagger = \Dcal^\mathrm{first} \cup \{(\vx_{n+1}, \vp_{n+1})\}$.
In our experiments, we refer to this approach as calibrated BPS with first-order data. \cref{tab:methods} provides an overview of all three cases, the corresponding set predictors, and the guarantees they deliver.

\begin{proposition}
\label{prop:pac:unknown}
Let a credal set predictor output $\gQ_i = \convex(\{\vpi_i^{(j)}\}_{j=1}^m)$ without any validity assumption, and let $\alpha,\beta \in (0,1)$ be arbitrary. 
Let $\lambda^\star$ be defined as in~(\ref{eq:lambda_CP:satisfaction}) using the first-order calibration data $\Dcal^\mathrm{first}$ that is exchangeable with the test point. 
Let $\gC_{\mathrm{BPS}}(\vx_i)$ be constructed as in~(\ref{eq:bps:sampling}) using $\vb_i^{\lambda^\star}$ from~(\ref{eq:optimal_BPS}). 
Then the resulting prediction sets satisfy
\begin{align*}
\Pr_{\gD_\dagger}\!\left[
\Pr\big(y_{n+1} \in \gC_{\mathrm{BPS}}(\vx_{n+1}) \mid \vx_{n+1}\big) \ge 1 - \alpha 
\right] 
\ge 1-\beta .
\end{align*}
\end{proposition}
This is again a PAC-style guarantee, similar to the one obtained in the case of partially valid credal sets. 
However, in contrast to that setting, the parameter $\beta$ is user specified and directly controls the tolerance for violations of the conditional coverage requirement.
Moreover, since this is a conformal guarantee, it is essentially tight in finite samples: under exchangeability, the coverage probability is also upper bounded by $1-\beta + \frac{2}{n+1}$, reflecting the usual sharpness of conformal bounds. \cref{sec:appendix:cond-to-marg} shows how these conditional guarantees translate into marginal ones.
% \paragraph{Remarks.}
\paragraph{Remark (calibration beyond BPS).}
Although we instantiated the calibration mechanism with BPS, the conformalization step is not restricted to this particular set predictor. In principle, for any set predictor $\gC_\lambda(\vx_i)$ parameterized by $\lambda$, one may define the risk as the indicator that the conditional coverage falls below $1-\alpha$. 
As long as this risk is non-increasing in $\lambda$, conformal risk control can be applied to determine a calibrated $\lambda^\star$ that yields the same distribution-free guarantee (see \cref{sec:appendix:alternative:set}). 
% In particular, even a probabilistic classifier combined with APS as a set predictor can be calibrated in this manner (see Appendix for experimental results).
\paragraph{Remark (data requirements and the zero-order proxy).}
The first-order data used in this case are required only for calibration and not for training. 
Credal set predictors are typically trained using standard zero-order data, and therefore, the calibration set does not need to be very large.
However, when no first-order calibration data (or reliable approximations thereof) are available, one possible approach is to use the one-hot encoding $\ve_{y_i}$ of the observed labels as a proxy for $\vp_i$ and treat these as first-order observations within the same calibration procedure. 
This yields an approximate version of the guarantee (see \cref{sec:appendix:alternative}). 
In our experiments, we refer to this variant as \emph{calibrated BPS with zero-order data}, in contrast to the original \emph{calibrated BPS with first-order data}.

\section{Experiments}\label{sec:experiments}
\begin{figure}[t!]
    \centering
    \includegraphics[width=0.85\columnwidth]{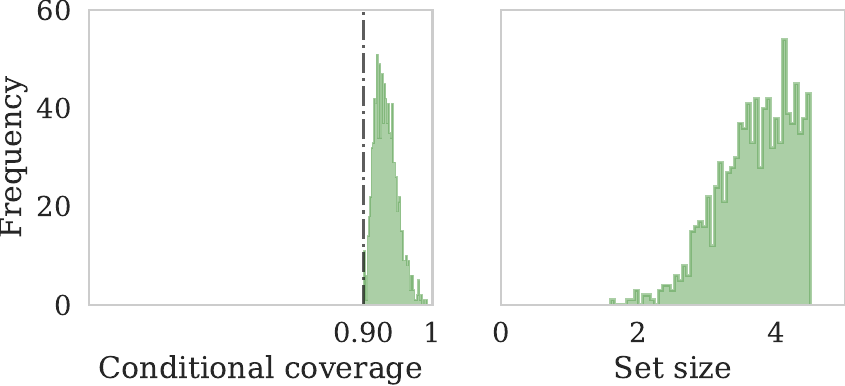}
    \caption{Histograms of the conditional coverage and (expected) set size of BPS$(1-\alpha)$ under valid credal sets. Here, $1-\alpha = 0.9$.}
    \label{fig:valid}
\end{figure}
In this section, we evaluate the performance of our proposed BPS under different scenarios. First, we consider a setting in which valid credal sets are available. Since such a setup does not yet exist in practice, we simulate it synthetically. Next, we study a scenario in which the credal sets are only partially valid. In these two cases, we simply apply BPS$(1-\alpha)$ without using conformal prediction. Finally, we turn to the setting where the validity of the credal sets is unknown. In this case, we apply BPS together with the proposed calibration procedure. All implementations and experiments can be found on our \href{https://github.com/alireza-javanmardi/conformal-BPS}{GitHub}
repository.
\footnote{{The link to the code: \url{https://github.com/alireza-javanmardi/conformal-BPS}}}

\paragraph{Datasets.} For the partially valid and unknown-validity settings, we conduct experiments on three real-world datasets for which first-order data is available. These include \emph{CIFAR-10} \citep{krizhevsky2009learning}, a 10-class image classification dataset, for which \emph{CIFAR-10H} \citep{peterson2019human} provides label distributions over the test set; \emph{ChaosNLI} \citep{nie2020what}, a 3-class natural language inference dataset; and \emph{QualityMRI} \citep{obuchowicz2020qualityMRI, schmarje2022benchmark}, a two-class medical image classification dataset. For all these datasets, we additionally use the corresponding zero-order data for evaluation purposes. Details regarding these datasets, the train–test–calibration splits, and information about the models (both the credal set predictors and the probabilistic classifiers) are provided in \cref{sec:appendix:exp-details}.
\paragraph{Metrics.} In all experiments, we consider three evaluation metrics. Since the main focus of the paper is conditional coverage, we report the conditional coverage satisfaction (\textbf{Cond. Sat.}), defined as the complement of the risk in (\ref{eq:risk:satisfaction}). We also report the (expected) \textbf{set size}, as defined in (\ref{eq:bps:set-size}), and the (expected) coverage of the realized label (\textbf{Marg. Cvg.}). All metrics are averaged over the test inputs.
\subsection{Valid credal sets}
To illustrate the performance of BPS on valid credal sets, we first generate $n = 1000$ instances, each represented by a convex hull of $m = 10$ randomly selected label distributions with $K = 5$ classes. For each instance, we sample a random convex combination of the $m$ distributions and treat it as the true label distribution. We then apply BPS$(1-\alpha)$ to each credal set and evaluate the conditional coverage and the set size accordingly. \cref{fig:valid} shows that the resulting prediction sets satisfy the conditional coverage property for every instance.
\subsection{Partially valid credal sets}\label{sec:experiments:partially-valid}
\begin{table}[t!]
    \centering
    \caption{Performance of BPS$(1-\alpha)$ on the partially valid credal sets for CIFAR-10. Here $1-\alpha = 0.9.$}
    \resizebox{\columnwidth}{!}{
        \begin{tabular}{l|cccc}
        \toprule
        % dataset & \multicolumn{4}{r}{cifar10} \\
        $\epsilon$ & Credal Cvg.  & Cond. Sat. & Marg. Cvg. & Set Size \\
        \midrule
        0.10 & 0.90 $\pm$ 0.01 & 0.97 $\pm$ 0.00 & 1.00 $\pm$ 0.00 & 6.55 $\pm$ 0.23 \\
        0.20 & 0.80 $\pm$ 0.01 & 0.93 $\pm$ 0.01 & 0.99 $\pm$ 0.00 & 1.72 $\pm$ 0.43 \\
        0.30 & 0.70 $\pm$ 0.01 & 0.84 $\pm$ 0.01 & 0.95 $\pm$ 0.01 & 1.18 $\pm$ 0.02 \\
        \bottomrule
        \end{tabular}
        }
    \label{tab:partially-valid:cifar10}
\end{table}
\begin{figure*}[ht!]
\centering
\includegraphics[width=0.9\textwidth]{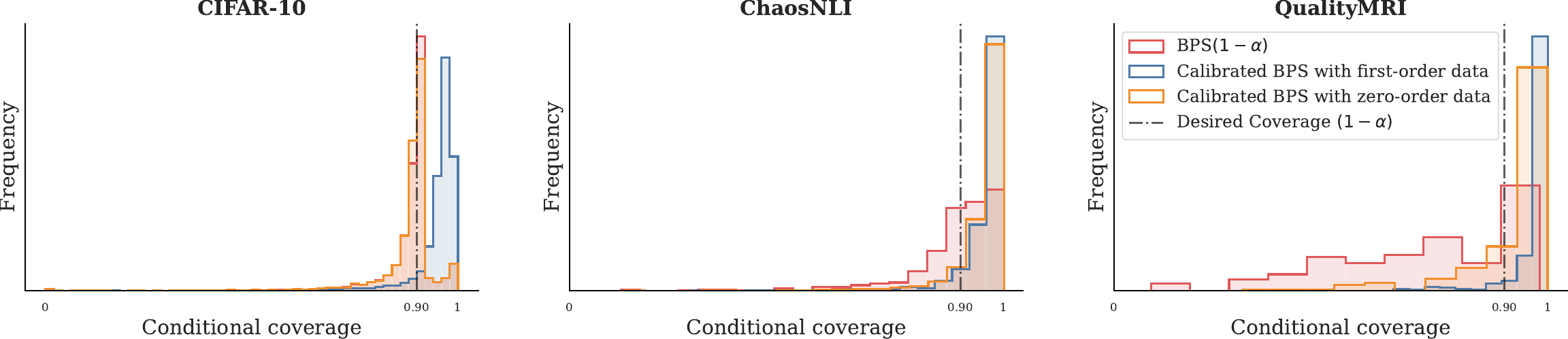}
\caption{Conditional coverage histograms of BPS$(1-\alpha)$ and calibrated BPS with zero- and first-order data on credal sets with unknown validity across three real-world datasets. $1-\alpha = 0.9$ and $1-\beta = 0.9$.}
\label{fig:unknown}
\end{figure*}

\begin{table*}[t!]
    \centering
		\caption{Performance comparison of BPS$(1-\alpha)$ and calibrated BPS with zero- and first-order data on credal sets with unknown validity across three real-world datasets. $1-\alpha = 0.9$ and $1-\beta = 0.9$. }
        \label{tab:real-world:BPS}
  \resizebox{\textwidth}{!}{
\begin{tabular}{l|ccc|ccc|ccc}
\toprule
&  \multicolumn{3}{c}{\textbf{CIFAR-10}} & \multicolumn{3}{|c}{\textbf{ChaosNLI}} & \multicolumn{3}{|c}{\textbf{QualityMRI}}\\
\cmidrule(r){2-4}
\cmidrule(r){5-7}
\cmidrule(r){8-10}
 Approach & Cond. Sat. & Marg. Cvg. & Set Size & Cond. Sat. & Marg. Cvg. & Set Size & Cond. Sat. & Marg. Cvg. & Set Size \\
\midrule
BPS $(1-\alpha)$ & 0.52 $\pm$ 0.00 & 0.90 $\pm$ 0.00 & 1.09 $\pm$ 0.00 & 0.61 $\pm$ 0.01 & 0.89 $\pm$ 0.01 & 2.00 $\pm$ 0.02 & 0.33 $\pm$ 0.05 & 0.81 $\pm$ 0.02 & 1.36 $\pm$ 0.04 \\
% BPS with marginal miscoverage risk & 0.34 $\pm$ 0.17 & 0.90 $\pm$ 0.00 & 1.09 $\pm$ 0.01 & 0.71 $\pm$ 0.05 & 0.91 $\pm$ 0.02 & 2.09 $\pm$ 0.05 & 0.66 $\pm$ 0.09 & 0.93 $\pm$ 0.03 & 1.71 $\pm$ 0.07 \\
% BPS with mean conditional miscoverage risk & 0.75 $\pm$ 0.02 & 0.93 $\pm$ 0.00 & 1.17 $\pm$ 0.01 & 0.70 $\pm$ 0.03 & 0.90 $\pm$ 0.01 & 2.08 $\pm$ 0.03 & 0.83 $\pm$ 0.05 & 0.97 $\pm$ 0.01 & 1.85 $\pm$ 0.04 \\
Calibrated BPS with first-order & 0.90 $\pm$ 0.01 & 0.98 $\pm$ 0.00 & 1.45 $\pm$ 0.05 & 0.91 $\pm$ 0.02 & 0.95 $\pm$ 0.01 & 2.40 $\pm$ 0.04 & 0.94 $\pm$ 0.04 & 0.99 $\pm$ 0.01 & 1.93 $\pm$ 0.02 \\
Calibrated BPS with zero-order & 0.52 $\pm$ 0.00 & 0.90 $\pm$ 0.00 & 1.09 $\pm$ 0.00 & 0.90 $\pm$ 0.01 & 0.95 $\pm$ 0.01 & 2.35 $\pm$ 0.04 & 0.77 $\pm$ 0.10 & 0.96 $\pm$ 0.03 & 1.81 $\pm$ 0.07 \\
\bottomrule
\end{tabular}

    }
\end{table*}
In order to obtain credal sets that are partially valid, we consider conformalized credal set predictors \citep{javanmardi2024conformalized}. To that end, we use a probabilistic classifier as the base learner together with a set of first-order calibration data $\Dcal^\mathrm{first} = \{(\vx_i, \vp_i)\}_{i = 1}^n$. On this data, we compute the total variation distance between the true label distribution $\vp_i$ and the estimated one $\vpi_i$, forming the score set $\gS_\text{TV} = \{\operatorname{TV}(\vpi_i, \vp_i)\}_{i=1}^{n}$.
For any future test input $\vx_{n+1}$, the conformalized credal set is then defined as
\begin{align*}
    \gQ_{n+1} = \{ \vq \in \triangle^K : \operatorname{TV}(\vpi_{n+1}, \vq) \leq \quantile{1 - \epsilon}{\gS_\text{TV}} \}.
\end{align*}
For every $\epsilon \in (0,1)$, this credal set is guaranteed to cover the true label distribution with high probability, i.e.,
\begin{align*}
    \Pr_{\gD_\dagger}[\vp_{n+1} \in \gQ_{n+1}] \ge 1 - \epsilon.
\end{align*}
In \cref{thm:tv-credal} in \cref{sec:tv-distance credal}, we provide an analytical derivation of the vertices of the credal sets constructed in this manner.
Given the credal sets defined as the convex hull of their vertices, we then apply BPS$(1-\alpha)$ on the test set. For this method, we randomly split the evaluation data into calibration and test subsets, repeat this procedure $10$ times with different random seeds, and report the results averaged over these runs.

In \cref{tab:partially-valid:cifar10}, we show the performance of this method on CIFAR-10 for $\epsilon \in \{0.1, 0.2, 0.3\}$. For this setting, we also report the coverage of the credal sets (\textbf{Credal Cvg.}), which matches the theoretical guarantee. As $\epsilon$ increases, the credal sets shrink in general, which results in smaller prediction sets. Furthermore, it can be seen that the conditional coverage satisfaction is in general higher than $1-\epsilon$, which is also expected, as explained in \cref{Sec:CaseII}. This is mainly because prediction sets constructed from invalid credal sets may also satisfy conditional coverage. We have also applied the same approach to the other two datasets and provide the corresponding results in \cref{tab:partially-valid:full} in \cref{sec:appendix:partial}.

\subsection{Credal sets with unknown validity}\label{sec:experiments:unknown}
For this case, we consider the Credal Ensembling (CreEns) approach \citep{nguyen2025credal} as a credal set predictor, where the credal set is constructed as the convex hull of the softmax outputs of an ensemble of neural networks (see \cref{subsec:appendix:models} for further details on the CreEns model). For each dataset, we randomly set aside a portion of the data for calibration and apply calibrated BPS with both zero- and first-order data. This random split into calibration and test sets is repeated $10$ times with different random seeds; prediction sets are constructed in each run, and the reported results are averaged over these seeds. In addition, we consider BPS$(1-\alpha)$ (i.e., BPS without calibration at nominal level $1-\alpha$) as a baseline.

\cref{fig:unknown} compares the conditional coverage over the test set for the three approaches on each dataset, and \cref{tab:real-world:BPS} provides the full performance comparison over the test sets for all three approaches and datasets. As expected, BPS calibrated with first-order data performs best, satisfying the desired conditional coverage for at least a $1-\beta=0.9$ fraction of cases. In contrast, BPS$(1-\alpha)$ violates conditional satisfaction in all cases, although it produces the most efficient (smallest) prediction sets. Calibrated BPS with zero-order data improves over BPS$(1-\alpha)$ in terms of conditional coverage; in some cases, such as ChaosNLI, it performs comparably to BPS calibrated with first-order data. Overall, it lies between the other two methods with respect to both efficiency and conditional coverage.

The results in all three cases support our theoretical findings. In \cref{sec:appendix:alternative} and \cref{sec:appendix:alternative:set}, we further investigate the effect of alternative calibration strategies—corresponding to different risk definitions—on these three datasets in the setting of credal sets with unknown validity. Moreover, since the guarantee in (\ref{eq:guarantee:unknown}) can be achieved by any set predictor and is not restricted to BPS, we additionally examine the performance of the proposed calibration strategy when applied to alternative set predictors. \cref{sec:appendix:ood} additionally evaluates all methods under distribution shift.
\section{Conclusion}
In this paper, we consider the problem of set prediction in a classification setting where the underlying model is a credal set predictor that outputs a set of label distributions for each input. We start with the case where the credal sets are valid, meaning the true label distribution lies within the credal set for every input. We propose a parametric set predictor, the Bernoulli prediction set (BPS), that outputs the smallest possible set guaranteeing conditional coverage at each input. When validity holds only with high probability over inputs, BPS can be applied to achieve a PAC-style conditional coverage guarantee; that is, with high probability, the achieved conditional coverage is at least the desired level. When the validity of the credal sets is unknown, BPS cannot provide guarantees on its own. For this case, we use first-order calibration data together with a risk function, defined as the indicator that conditional coverage is below the desired level, to calibrate BPS through the Conformal Risk Control (CRC) framework. This approach also leads to a PAC-style conditional coverage guarantee. In this case, the amount of violation of conditional coverage can be controlled by the user-specified error rate of CRC.

% \begin{contributions} % will be removed in pdf for initial submission 
% 					  % (without ‘accepted’ option in \documentclass)
%                       % so you can already fill it to test with the
%                       % ‘accepted’ class option
%     Briefly list author contributions. 
%     This is a nice way of making clear who did what and to give proper credit.
%     This section is optional.

%     H.~Q.~Bovik conceived the idea and wrote the paper.
%     Coauthor One created the code.
%     Coauthor Two created the figures.
% \end{contributions}

\begin{acknowledgements} 
Alireza Javanmardi was supported by the Klaus Tschira Stiftung (project 00.019.2024). Willem Waegeman was supported by the Flemish Government under the Flanders AI Research program. The authors are thankful to Stefan Heid for his valuable insights on \cref{thm:tv-credal}. 
\end{acknowledgements}

% References
% \clearpage
\bibliography{UAI/references}

\newpage

\onecolumn

% \title{Optimal Conformal Prediction under Epistemic Uncertainty\\(Supplementary Material)}
% \maketitle
\appendix
% Numbering / display (your existing setup — these are fine)
\renewcommand{\thesection}{\Alph{section}}
\renewcommand{\thesubsection}{\Alph{section}.\arabic{subsection}}
\titleformat{\section}
  {\normalfont\large\bfseries}{Appendix \thesection:}{1em}{}

% --- cleveref part: retype, don't rename ---
\crefalias{section}{appendix}
\crefalias{subsection}{appendix}
\crefname{appendix}{Appendix}{Appendices}
\Crefname{appendix}{Appendix}{Appendices}

\section{Related Work}\label{sec:related_work}
\paragraph{Set prediction under epistemic uncertainty.} The connection between conformal prediction and epistemic uncertainty-aware predictors has recently attracted increasing attention. \citet{rossellini2024integrating} incorporate epistemic uncertainty into conformalized quantile regression to improve conditional coverage, while \citet{azizi2026clear} construct prediction sets by adaptively balancing epistemic and aleatoric uncertainty. Our work differs from these two as we focus on classification, and their regression-based approaches are not directly transferable. Furthermore, these methods often rely on heuristics that do not provide the same theoretical guarantees as our framework.
\citet{karimi2024evidential} introduce a nonconformity score based on uncertainty estimates derived from evidential models, and \citet{cabezas2025epistemic} integrate epistemic uncertainty into conformal prediction by fitting a Bayesian model on top of nonconformity scores. What separates our work from these Bayesian and evidential approaches is that we study how epistemic uncertainty, represented explicitly via credal sets, can be incorporated into conformal prediction. Our focus is on providing finite-sample conditional coverage guarantees, whereas previous methods lack such formal properties in this setting.
The problem of deriving prediction sets from credal sets has also been studied in the imprecise probability literature \citep{caprio2024credalbayesian}. However, such approaches typically do not provide finite-sample coverage guarantees without additional structural assumptions.

\paragraph{Availability of first-order data.} In practice, many datasets are labeled by human annotators, and in many cases, multiple annotators are involved. This approach is becoming increasingly common, with crowdsourcing tools becoming an integral component of dataset collection, making multiple annotations per data instance more accessible \citep{kovashka2016crowdsourcing, sorokin2008utility}. Disagreement among annotators is also very common, especially in natural language tasks \citep{abercrombie2023consistency, nie2020what} and computer vision \citep{beyer2020we, peterson2019human, schmarje2022benchmark}. This facilitates the availability of such data in practice. In our case, training the credal set predictor relies solely on standard zero-order data; first-order data is only required during the calibration step and need not be large. The number of calibration samples primarily affects the sharpness of the coverage distribution in conformal prediction, and in practice, tens to a few hundred such samples are often sufficient.
\section{Proofs}\label{sec:appendix:proofs}
\paragraph{\cref{prop:optimal}}
\begin{proof}
    We show that this set achieves at least $1 - \alpha$ expected conditional coverage with respect to any $\vp \in \gQ_i$. Since $\vp$ lies in the convex hull of $\{\vpi_i^{(j)}\}_{j=1}^m$, we can write $\vp = \sum_{j = 1}^m \eta_j \vpi_i^{(j)}$ for some $\eta_j \in [0, 1]$ with $\sum_{j = 1}^m \eta_j = 1$. It follows that $\vb_i^{1 - \alpha} \cdot \vp = \sum_{j = 1}^m \eta_j \, (\vb_i^{1 - \alpha} \cdot \vpi_i^{(j)})$, and by definition, $\vb_i^{1 - \alpha} \cdot \vpi_i^{(j)} \ge {1 - \alpha}$ for all $j$.
 Then
    \begin{align*}
        \vb_i^{1 - \alpha} \cdot\vp = \sum_{j = 1}^m \eta_j (\vb_i^{1 - \alpha}  \cdot \vpi_i^{(j)}) \ge {1 - \alpha}\sum_{j = 1}^m \eta_j  = {1 - \alpha}.
    \end{align*}
    The fact that it is the minimal set directly follows from the definition of \cref{eq:optimal_BPS}. 
\end{proof}
\paragraph{\cref{prop:EUadaptive}}
\begin{proof}
    Since $\vp_i \in \gQ_i \subseteq \gQ_i'$, the solution to the optimization in \cref{eq:optimal_BPS} with $\gQ'_i$ is also feasible for the same optimization with $\gQ_i$. Therefore, $\vb_i^\lambda \cdot \boldsymbol{1} 
\;\le\; 
\vb_i'^\lambda \cdot \boldsymbol{1}$. 
\end{proof}
\paragraph{\cref{prop:equivalence}}
\begin{proof}
For a singleton credal set $\gQ_i = \{\vpi_i\}$, the optimization problem 
in \cref{eq:optimal_BPS} reduces to
\begin{align*}
\vb_i^\lambda
=
\argmin_{\vb} \ \vb \cdot \boldsymbol{1}
\quad
\text{s.t.}
\quad
\vb \cdot \vpi_i \ge \lambda .
\end{align*}
This is a fractional knapsack problem. Its optimal solution is obtained by 
sorting $\vpi_i$ in decreasing order and allocating inclusion probabilities 
to the corresponding labels greedily.
Let $\tilde{\vpi}_i$ denote the sorted version of $\vpi_i$ such that 
$\tilde{\pi}_{i1} \ge \dots \ge \tilde{\pi}_{iK}$, and let 
$L(\vpi_i,\lambda) =
\min \{ k \in [K] : \sum_{j=1}^{k} \tilde{\pi}_{ij} \ge \lambda \}$. Then the optimal vector in the sorted coordinate 
system is
$(1,\dots,1, b_{iL(\vpi_i,\lambda)}^\lambda, 0,\dots,0)$,
where the first $L(\vpi_i,\lambda)-1$ entries are equal to $1$, the entries after $L(\vpi_i,\lambda)$ are 
$0$, and
\[
b_{iL(\vpi_i,\lambda)}^\lambda
=
\frac{\lambda - \sum_{c=1}^{L(\vpi_i,\lambda)-1} \tilde{\pi}_{ic}}
{\tilde{\pi}_{iL(\vpi_i,\lambda)}} .
\]
Reordering back to the original label indexing yields $\vb_i^\lambda$, 
which satisfies $\vb_i^\lambda\cdot \vpi_i = \lambda$.

This coincides exactly with the Bernoulli inclusion vector underlying APS 
in (\ref{eq:APS_b}) for the same $\vpi_i$ and threshold $\lambda$. In both 
cases, the top $L(\vpi_i,\lambda)-1$ labels are included deterministically, 
the $L(\vpi_i,\lambda)$-th label is included with probability $b_{iL(\vpi_i,\lambda)}^\lambda$, 
and all remaining labels are excluded. Hence, the BPS and APS constructions 
induce the same inclusion probabilities and therefore yield equivalent 
prediction sets in expectation.
\end{proof}
\paragraph{\cref{prop:pac:partially-valid}}
\begin{proof}
    On the event ${\vp_i \in \gQ_i}$, validity of the credal set implies $\vb_i^{1-\alpha} \cdot \vp_i \ge 1-\alpha$ by construction of BPS. Since this event occurs with probability at least $1-\epsilon$, the result follows.
\end{proof}
\paragraph{\cref{prop:pac:unknown}}
\begin{proof}
By construction, $\lambda^\star$ controls the empirical risk in~(\ref{eq:risk:satisfaction}). 
The result follows directly from the distribution-free guarantee of conformal risk control, which ensures that the expected risk of the test point is bounded by $\beta$ under exchangeability.
\end{proof}
\section{From Conditional to Marginal Coverage Guarantees}\label{sec:appendix:cond-to-marg}
The conditional coverage guarantee directly implies a marginal coverage guarantee as well. In particular, if
$\vb_{i}^{1-\alpha} \cdot \vp_{i} \ge 1-\alpha, \quad \forall \vx_{i}$,
then automatically $\Pr\big[y_{n+1} \in \gC_{\mathrm{BPS}}(\vx_{n+1})\big] \ge 1-\alpha$, where the probability is over the test point and the randomization of the set.
% However, in the PAC-style conditional coverage setting, where
% $\Pr_{\gD_\dagger}\left[\vb_{n+1}^{\lambda^\star}\cdot \vp_{n+1} \ge 1-\alpha\right] \ge 1-\beta$, the guarantee transfers to a marginal coverage guarantee of at least $(1-\beta)(1-\alpha)$. Indeed,
However, in the PAC-style conditional coverage setting, the guarantee 
transfers to a marginal coverage guarantee of at least 
$(1-\delta)(1-\alpha)$, where $\delta$ denotes the corresponding error 
rate ($\delta = \epsilon$ under partially valid credal sets in 
\cref{Sec:CaseII}, and $\delta = \beta$ under credal sets with unknown 
validity in \cref{sec:CaseIII}). We show this for Case~III; the argument 
for Case~II is identical, with the probability taken over the random 
draw of the input instead of $\gD_\dagger$. Indeed,
\begin{align*}
\Pr_{\gD_\dagger}\big[y_{n+1} \in \gC_{\mathrm{BPS}}(\vx_{n+1})\big]
= \E_{\gD_\dagger}\big[\Pr(y_{n+1} \in \gC_{\mathrm{BPS}}(\vx_{n+1}) \mid \vx_{n+1}) \big]
= \E_{\gD_\dagger}\big[\vb_{n+1}^{\lambda^\star}\cdot \vp_{n+1}\big] 
\ge (1-\beta)(1-\alpha) + \underbrace{\beta \cdot 0}_{\text{worst case}}.
\end{align*}
In practice, however, we observe that the achieved marginal coverage under the PAC-style conditional guarantee is often substantially higher than this worst-case bound.
\section{TV-Distance Neighborhood: A Specific Credal Set with Analytical Corner Points}\label{sec:tv-distance credal}
In this section, we show that a credal set defined as the set of all label distributions in the simplex $\triangle^K$ with total variation distance less than or equal to $d$ from $\vp$ can be analytically described by a set of $K(K-1)$ corner points.
\begin{theorem}\label{thm:tv-credal}
     The $d$-neighborhood of a distribution $\vp \in \triangle^K$, i.e., the set of all distributions in $\triangle^K$ with a total variation distance $\leq d$, is a polytope defined by a finite number of corner points in $\triangle^K$.   
\end{theorem}
\begin{proof}
Let $\bar{\triangle}^K \supset \triangle^K$ denote the set of all $\vp = (p_1, \ldots, p_K) \in \mathbb{R}^K$ such that $p_1 + \ldots + p_K = 1$ (i.e., also including negative entries). The set of all $\vq \in \bar{\triangle}^K $ whose total variation distance from $\vp$ is bounded by $d$, i.e., 
$$
N_d(\vp) := \left\{ \vq \in \bar{\triangle}^K \, \vert \, \max_{A \subseteq [K]} \left| 
\sum_{k \in A}  p_k - \sum_{k \in A} q_k \right| \leq d  \right\} \, ,
$$
is given by the convex polytope $C$ with corner points $\vp^{i,j} = \vp + \vh^{i,j}$, $1 \leq i \neq j \leq K$, where the $i^{th}$ entry in $\vh^{i,j}$ is $+d$, the $j^{th}$ entry is $-d$, and all other entries are 0. For the direction $C \subset N_d(\vp)$, note that any $\vp^{i,j}$ is obviously in $N_d(\vp)$. Moreover, for 
$\vq = \sum_{i,j} \alpha_{i,j} \, \vp^{i,j}$, where $\alpha_{i,j} \geq 0$ and $\sum_{i,j} \alpha_{i,j} = 1$, and any $A \subseteq [K]$, 
\begin{align*}
   \sum_{k \in A}  p_k - \sum_{k \in A} q_k & =  \sum_{k \in A}  p_k - \sum_{k \in A} \sum_{i,j} \alpha_{i,j} \, p^{i,j}_k \\
   & = \sum_{k \in A} \sum_{i,j} \alpha_{i,j} \, (p_k - p^{i,j}_k ) \\
   & = \sum_{i,j} \alpha_{i,j} \sum_{k \in A} \, (p_k - p^{i,j}_k ) \\
   & \leq \max_{i,j} \sum_{k \in A} \, (p_k - p^{i,j}_k ) \\
   & \leq d.
\end{align*}
To show the direction $N_d(\vp) \subset C$, consider any $\vq \in N_d(\vp)$ such that $\vq \neq \vp$. Let $A \subseteq [K]$ be a (smallest) subset that determines the total variation distance between $\vq$ and $\vp$. Then either $q_i > p_i$ for all $i \in A$, or $q_i < p_i$ for all $i \in A$. Consider the first case (without loss of generality) and let $(\cdot)$ be a permutation of $[K]$ such that $q_{(1)} - p_{(1)} \geq q_{(2)} - p_{(2)} \geq \ldots \geq q_{(K)} - p_{(K)}$. Then $A = \{ (1), \ldots , (J) \}$ for some $1 \leq J < K$ and
$$
\sum_{k \in [J]} q_{(k)} - p_{(k)} = d' \leq d \, .
$$
We can then write $q_{(k)} = \sum_{i,j} \alpha_{i,j} \, p^{i,j}_{(k)}$ for suitably chosen $\alpha_{i,j} \geq 0$ with $\sum_{i,j} \alpha_{i,j} = 1$, which means that $\vq \in C$. The $\alpha_{i,j}$ can be specified in a constructive way by processing the entries in $\vq$ one by one, starting with the $k$ for which the difference $|q_{(k)} - p_{(k)}|$ is smallest and proceeding to the larger ones. 

Now, the $d$-neighborhood of $\vp$ that we are looking for is given by the intersection $C \cap \Delta^K$, i.e., by the intersection of two convex polytopes. Thus, it is itself again a convex polytope. 
% {\color{red} WRONG
% Moreover, for a corner point $\vp^{i,j}$ that lies outside the simplex, the corresponding point on the line connecting $\vp$ to $\vp^{i,j}$ that is closest to $\vp^{i,j}$ but still lies in the simplex can be found as $\bar{\vp}^{i,j} = \vp + \eta \vh^{i,j}$, where $\eta = \min\left(1, \frac{1 - p_i}{d}, \frac{p_j}{d}\right)$.}
\end{proof}
\section{Toy Examples Illustrating Epistemic Adaptivity}\label{sec:appendix:toy}
\begin{figure}
\begin{subfigure}[c]{0.6\textwidth}
    \centering
    \includegraphics[width=\columnwidth]{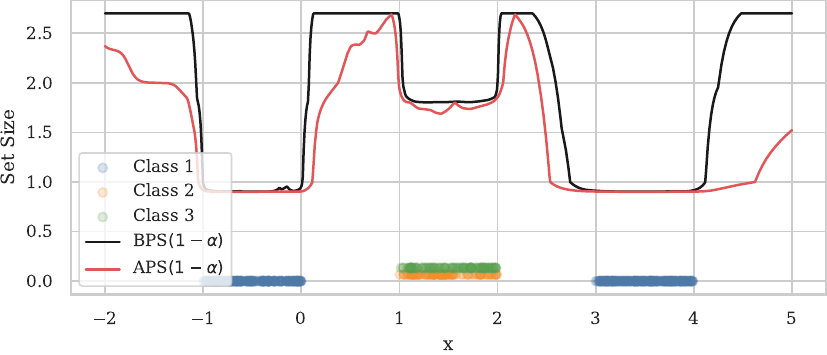} 
    \caption{Example 1}
    \label{fig:toy:BPSvsAPS:ex1}
\end{subfigure}
\hfill
\begin{subfigure}[c]{0.29\textwidth}
    \centering
    \includegraphics[width=\columnwidth]{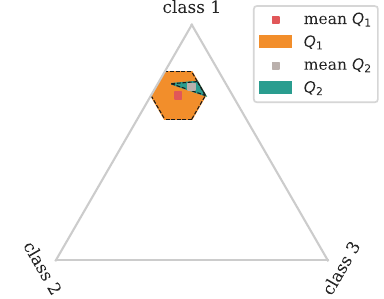} 
    \caption{Example 2}
    \label{fig:toy:BPSvsAPS:ex2}
\end{subfigure}
    % \centering
    % \includegraphics[width=0.6\columnwidth]{figs/toy_example/ss_comparison.pdf}
    \caption{Adaptivity of BPS versus APS to epistemic uncertainty. For both examples, we set $1-\alpha = 0.9$.}
    \label{fig:toy:BPSvsAPS}
\end{figure}
\paragraph{Example 1.} Consider the following toy example of a three-class classification problem, where $\vx$ is a one-dimensional input. The data-generating mechanism is defined as follows:
\begin{itemize}
    \item If $\vx_i \in [-1,0] \cup [3,4]$, then $y_i$ belongs to class 1.
    \item If $\vx_i \in [1,2]$, then $y_i$ belongs to class 2 or 3 with equal probability, i.e., $\Pr(y_i=2 \mid \vx_i)=\Pr(y_i=3 \mid \vx_i)=0.5$.
\end{itemize}
In other regions of the input space, namely $[-2,-1]$, $[0,1]$, $[2,3]$, and $[4,5]$, no training data are observed. An ensemble of neural networks is trained on this data and treated as a credal set predictor. For inputs $\vx_i \in [-2,5]$, prediction sets are constructed using BPS$(1-\alpha)$.
For comparison, we also aggregate the ensemble by taking the mean of its predictive distributions, treat it as a single probabilistic classifier, and construct prediction sets using APS$(1-\alpha)$.
As shown in \cref{fig:toy:BPSvsAPS:ex1}, in regions where sufficient training data are available, both set predictors behave similarly. However, in regions with no observed data, BPS reflects the increased epistemic uncertainty through larger prediction sets, whereas APS does not consistently account for this uncertainty.
\paragraph{Example 2.}
Consider two credal sets as illustrated in \cref{fig:toy:BPSvsAPS:ex2}, where $\gQ_2 \subset \gQ_1$. As a consequence of \cref{prop:EUadaptive} (epistemic adaptivity), the prediction set constructed by BPS for $\gQ_1$ must be at least as large as the one constructed for $\gQ_2$, since $\gQ_1$ represents higher epistemic uncertainty.
In this example, we also demonstrate that ignoring the credal set and instead working with the averaged predictor combined with APS does not satisfy this adaptivity property. Concretely, for $\gQ_1$, the Bernoulli inclusion vector obtained by APS (applied to the mean predictor) is $[1,1,0]$, while BPS yields $[1, 0.91, 0.58]$. For $\gQ_2$, APS produces $[1,1,0.23]$, whereas BPS results in $[1, 0.57, 0.71]$.
The expected set size (i.e., the sum of inclusion probabilities) for BPS decreases from $2.49$ to $2.28$ when moving from $\gQ_1$ to $\gQ_2$, reflecting the reduction in epistemic uncertainty. In contrast, for APS, the expected set size increases from $2$ to $2.23$, violating epistemic adaptivity.

\section{Experiments Details}\label{sec:appendix:exp-details}
\subsection{Models}\label{subsec:appendix:models} 
For the experiments on real-world datasets, we consider two credal set predictors: Credal Relative Likelihood (CreRL) \citep{lohr2025credal} and Credal Ensembling (CreEns) \citep{nguyen2025credal}. Both methods are based on an ensemble of neural networks. In all experiments, the number of ensemble members is fixed to $m=20$.
The CreRL framework considers a parameter $\gamma$ (referred to as $\alpha$ in the original paper; to avoid confusion with our error rate $\alpha$, we denote it by $\gamma$). It constructs the credal set from the $m$ ensemble members, where the relative likelihood of the $i^\text{th}$ member is limited to 
$\gamma + (i-1) \cdot \frac{1-\gamma}{m-1},  \forall i \in [m]$ by an early stopping strategy. 
The CreEns framework builds on standard ensemble training with $m$ members. It first considers the mean predictor as a representative and computes the Euclidean distance of each of the $m$ members to this representative. For a given $\gamma$, it then removes a $\gamma$ fraction of the ensemble members with the largest Euclidean distance from the representative.
All training details follow \cite{lohr2025credal} and their 
\href{https://github.com/timoverse/credal-prediction-relative-likelihood}{GitHub} repository\footnote{https://github.com/timoverse/credal-prediction-relative-likelihood}. In \cref{fig:real:credal-set-coverage}, we plot the coverage of the credal sets, defined as the convex hull of the $m$ predictors constructed by these methods, across three datasets for various values of $\gamma$. Note that the smaller the value of $\gamma$, the larger the resulting credal sets, hence the higher the coverage.

For the experiments in \cref{sec:experiments:partially-valid}, since we required a probabilistic classifier, we considered the average predictor of CreRL ($\gamma = 1.0$) ensemble as the probabilistic classifier. We choose $\gamma = 1.0$ because in this setting the coverage of the credal set is not relevant; we only rely on the average predictor as a probabilistic classifier.
For the experiments in \cref{sec:experiments:unknown}, in order to balance credal coverage and set size, we used CreEns ($\gamma = 0.4$), resulting in an ensemble of size $12$. Later, in \cref{tab:real-world:full}, when reporting experiments with APS, we use the average predictor of this ensemble as the probabilistic classifier.

\begin{figure}
\centering
\includegraphics[width=0.9\textwidth]{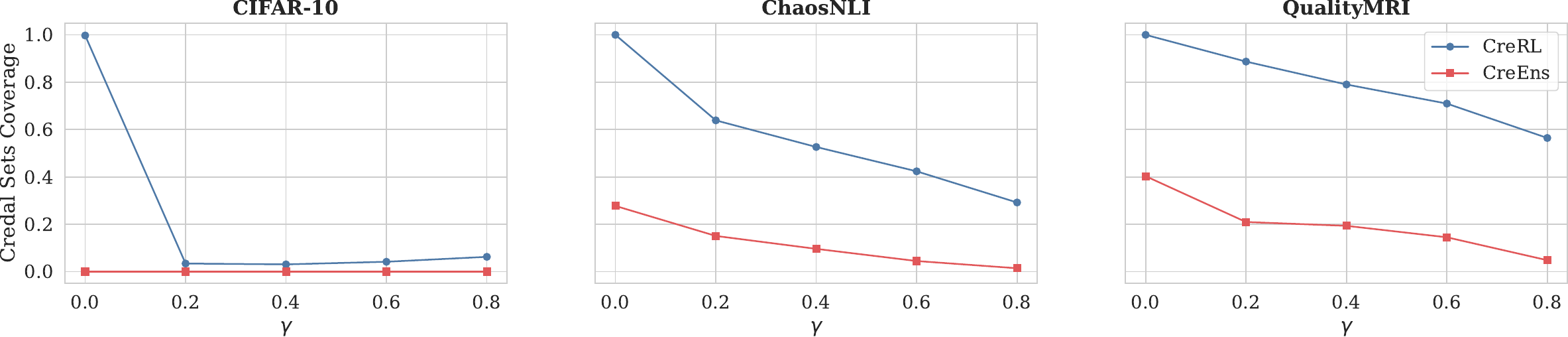}
\caption{Credal sets coverage for different credal set predictors on the three real-world datasets.}
\label{fig:real:credal-set-coverage}
\end{figure}

\subsection{Datasets}\label{subsec:appendix:datasets}
\paragraph{CIFAR-10.}
The CIFAR-10 dataset \citep{krizhevsky2009learning} is an image classification dataset with 10 classes, containing 50K training and 10K test instances. A variant of this dataset, namely CIFAR-10H \citep{peterson2019human}, provides multiple human annotations for each image in the CIFAR-10 test set. The human judgments for each image are aggregated into a categorical distribution over the classes, which serves as the oracle label distribution for that image. In all experiments with this dataset, the training data was used to train the models, while the test data was divided into $20\%$ for calibration and $80\%$ for evaluation.
As a neural network, we use the PyTorch implementation of ResNet-18, similar in spirit to the setup of \citet{lohr2025credal}.

\paragraph{ChaosNLI.}
ChaosNLI \citep{nie2020what} is an English Natural Language Inference (NLI) dataset in which the task is to classify the textual alignment between a given premise–hypothesis pair into three classes: entailment, contradiction, and neutral. The instances are selected from the development sets of SNLI \citep{bowman2015large}, MNLI \citep{williams2018broad}, and AbductiveNLI \citep{bhagavatula2019abductive}, focusing on examples with high disagreement among labels provided by human annotators. Each premise–hypothesis pair was then annotated by $100$ independent humans under strict annotation guidelines, and their responses were aggregated to form the oracle label distribution. We combine the Chaos-SNLI and Chaos-MNLI subsets, resulting in a dataset of $3113$ datapoints, of which $80\%$ is used for model training, $10\%$ for calibration, and $10\%$ for evaluation purposes. 
The premise–hypothesis pairs are embedded in the same manner as in \citet{javanmardi2024conformalized}, and a similar fully-connected neural network with $4$ hidden layers was used as the base for the ensemble.

\paragraph{QualityMRI.}
The QualityMRI dataset contains human magnetic resonance images (MRI) with varying levels of diagnostic quality \citep{obuchowicz2020qualityMRI}. Each image is evaluated independently by several radiologists, capturing variability in perceived image quality. For our purposes, the task is framed as a binary classification problem, where the aggregated radiologist annotations are interpreted as an oracle label distribution over the two classes \citep{schmarje2022benchmark}. The dataset is relatively small, comprising 310 instances in total. We use $80\%$ of the data for training, $10\%$ for calibration, and the remaining $10\%$ for evaluation. On this data, we employ the ResNet-18 architecture from the PyTorch torchvision package without pretrained weights as a base for the ensemble.

\section{Extension of Experiments with Partially Valid Credal Sets}\label{sec:appendix:partial}
In this section, we extend the results of \cref{sec:experiments:partially-valid} to the remaining two datasets, namely, ChaosNLI and QualityMRI. \cref{tab:partially-valid:full} confirms that what we observed for the case of CIFAR-10 also holds for the other two datasets. Specifically, the conditional coverage satisfaction is always greater than $1-\epsilon$.
\begin{table}[t!]
    \centering
    \caption{Performance of BPS$(1-\alpha)$ on the partially valid credal sets for different datasets. Here $1-\alpha = 0.9.$}
    \resizebox{\columnwidth}{!}{
        \begin{tabular}{l|cccc|cccc|cccc}
        \toprule
         &  \multicolumn{4}{c}{\bfseries CIFAR-10} & \multicolumn{4}{|c}{\bfseries ChaosNLI} & \multicolumn{4}{|c}{\bfseries QualityMRI} \\
         \cmidrule(r){2-5}
         \cmidrule(r){6-9}
         \cmidrule(r){10-13}
          $\epsilon$& Credal Cvg. & Cond. Sat. & Marg. Cvg. & Set Size & Credal Cvg. & Cond. Sat. & Marg. Cvg. & Set Size & Credal Cvg. & Cond. Sat. & Marg. Cvg. & Set Size \\
        \midrule
         0.10 & 0.90 $\pm$ 0.01 & 0.97 $\pm$ 0.00 & 1.00 $\pm$ 0.00 & 6.55 $\pm$ 0.23 & 0.89 $\pm$ 0.02 & 0.98 $\pm$ 0.01 & 0.92 $\pm$ 0.00 & 2.69 $\pm$ 0.00 & 0.90 $\pm$ 0.08 & 1.00 $\pm$ 0.00 & 0.90 $\pm$ 0.00 & 1.80 $\pm$ 0.00 \\
        0.20 & 0.80 $\pm$ 0.01 & 0.93 $\pm$ 0.01 & 0.99 $\pm$ 0.00 & 1.72 $\pm$ 0.43 & 0.79 $\pm$ 0.03 & 0.95 $\pm$ 0.01 & 0.93 $\pm$ 0.00 & 2.67 $\pm$ 0.00 & 0.79 $\pm$ 0.08 & 1.00 $\pm$ 0.00 & 0.90 $\pm$  0.00 & 1.80 $\pm$ 0.00 \\
        0.30 & 0.70 $\pm$ 0.01 & 0.84 $\pm$ 0.01 & 0.95 $\pm$ 0.01 & 1.18 $\pm$ 0.02 & 0.70 $\pm$ 0.04 & 0.92 $\pm$ 0.01 & 0.94 $\pm$ 0.00 & 2.63 $\pm$ 0.01 & 0.71 $\pm$ 0.10 & 1.00 $\pm$ 0.00 & 0.90 $\pm$  0.00 & 1.80 $\pm$ 0.00 \\
        \bottomrule
        \end{tabular}
        }
    \label{tab:partially-valid:full}
\end{table}

\section{Alternative Risk Functions for Calibration}\label{sec:appendix:alternative}
In \cref{sec:CaseIII}, when introducing our calibration approach, we specifically considered the risk in (\ref{eq:risk:satisfaction}), namely the indicator that the conditional coverage falls below the target level $1-\alpha$. We chose this risk in order to provide a PAC-style conditional coverage guarantee for our set predictor. However, CRC can be applied with arbitrary risk functions, leading to different types of guarantees.
Recall that $\Dcal^\mathrm{first} = \{(\vx_i, \vp_i)\}_{i=1}^n$ denotes the first-order calibration data, exchangeable with the future test point, and $\gD_\dagger = \Dcal^\mathrm{first} \cup \{(\vx_{n+1}, \vp_{n+1})\}$. Likewise, $\Dcal^\mathrm{zero} = \{(\vx_i, y_i)\}_{i=1}^n$ denotes the zero-order calibration data, with $\gD_+ = \Dcal^\mathrm{zero} \cup \{(\vx_{n+1}, y_{n+1})\}$.
When first-order calibration data are available, one may consider the conditional miscoverage risk
\[
\gL(\vx_i, \lambda) = 1 - \vb_i^\lambda \cdot \vp_i.
\]
Alternatively, when only zero-order data are accessible, one may use the marginal miscoverage risk
\[
\gL(\vx_i, \lambda) = 1 - \vb_i^\lambda \cdot \ve_{y_i},
\]
where $\ve_{y_i}$ denotes the one-hot encoding of $y_i$. This latter risk can be viewed as a soft extension of classical conformal prediction with miscoverage risk $\mathds{1}(y_i \notin \gC(\vx_i))$ to the randomized set setting. In \cref{tab:risk-guarantees}, we provide an overview of the different risk functions, the data they require, and the guarantees they deliver. In \cref{tab:real-world:full}, we compare the performance of calibrated BPS under these risk formulations.
\begin{table}[t]
\renewcommand{\arraystretch}{1.2}
    \centering
    \caption{Different risk functions, their calibration data requirements, and their guarantees.}
    \begin{tabular}{cccc}
    \toprule
    Approach Name & Risk $\gL(\vx_i, \lambda)$ & Calibration Data & Guarantee \\ 
    \midrule
       CondSatFirst& $\mathds{1}\left[\vb_i^\lambda \cdot \vp_i < 1-\alpha\right]$ & first-order data  & $\Pr_{\gD_\dagger}\left[\vb_{n+1}^{\lambda^\star} \cdot \vp_{n+1} \ge 1-\alpha\right] \geq 1-\beta$ \\ 
       CondSatZero& $\mathds{1}\left[\vb_i^\lambda \cdot \ve_{y_i} < 1-\alpha\right]$ & zero-order data  & $\Pr_{\gD_+}\left[\vb_{n+1}^{\lambda^\star} \cdot \ve_{y_{n+1}} \ge 1-\alpha\right] \geq 1-\beta$ \\ 
       MeanCondFirst& $1- \vb_i^\lambda \cdot \vp_{i}$ & first-order data  & $\mathbb{E}_{\gD_\dagger}\left[\vb_{n+1}^{\lambda^\star} \cdot \vp_{n+1} \right] \geq 1-\beta $ \\ 
       MargZero& $1- \vb_i^\lambda \cdot \ve_{y_i}$ & zero-order data  & $\mathbb{E}_{\gD_+}\left[\vb_{n+1}^{\lambda^\star} \cdot \ve_{y_{n+1}} \right] \geq 1-\beta $ \\ 
       \bottomrule
    \end{tabular}
    
    \label{tab:risk-guarantees}
\end{table}

\section{Alternative Set Predictors}\label{sec:appendix:alternative:set}
\begin{figure}
    \centering
    \includegraphics[width=0.6\columnwidth]{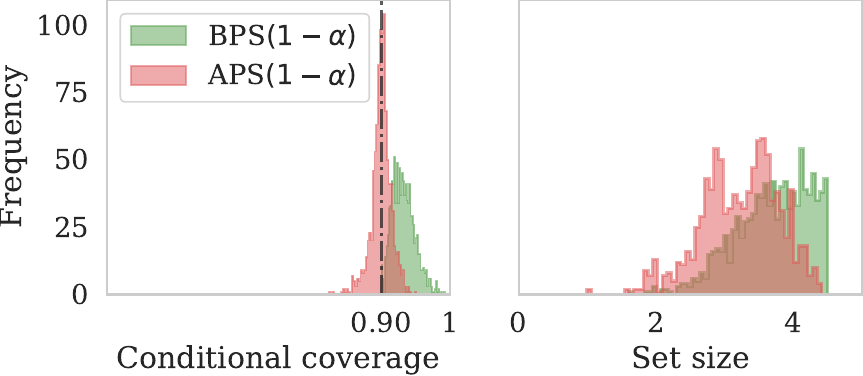}
    \caption{Histograms of the conditional coverage and (expected) set size of BPS$(1-\alpha)$ and APS$(1-\alpha)$ under valid credal sets. Here, $1-\alpha = 0.9$.}
    \label{fig:valid:BPSvsAPS}
\end{figure}
While in our experiments we used BPS as a natural base set predictor built on top of a credal set predictor, one may ask whether alternative set predictors could be employed. To this end, recall that in the fully valid and partially valid credal set settings (\cref{Sec:CaseI} and \cref{Sec:CaseII}), we do not apply any conformal prediction framework. In these cases, BPS alone delivers the desired (exact or PAC-style) conditional coverage guarantee with respect to the true label distribution. Hence, it is not immediately clear what would be gained by replacing BPS with another set predictor.

However, for the case of credal sets with unknown validity (\cref{sec:CaseIII}), as we apply CRC to get the PAC-style conditional coverage guarantee, any other set predictor can be used together with the risk defined in (\ref{eq:risk:satisfaction}) and the first-order calibration data to deliver the guarantee. For example, one could aggregate the credal set into a single representative distribution, for instance, by taking the mean over the $m$ distributions defining the credal set, and treat it as a probabilistic classifier. On top of this averaged predictor, methods such as APS could then be applied. While this is a valid construction, it effectively collapses the credal set into a single distribution and therefore no longer explicitly represents epistemic uncertainty. As a consequence, even if such a method satisfies a similar coverage guarantee after calibration, the resulting prediction sets do not adapt to epistemic uncertainty in the same principled way as BPS derived directly from the credal set.

To further investigate this point, we compare the performance of BPS and APS in all three cases. For Cases I and II (valid and partially valid credal sets), we simply replace BPS$(1-\alpha)$ with APS$(1-\alpha)$, i.e., APS with the same nominal coverage level. \cref{fig:valid:BPSvsAPS} compares the conditional coverage and set sizes for the case of valid credal sets. It can be seen that, unlike BPS$(1-\alpha)$, APS$(1-\alpha)$ ignores the information contained in the credal sets and violates conditional coverage in a substantial fraction of cases, while producing more efficient (smaller) prediction sets. \cref{tab:partially-valid:full:APS} reports the corresponding results for partially valid credal sets. Note that, in this setting, APS is applied to the mean of the ensemble, which is independent of $\epsilon$; hence, unlike BPS, its results do not vary with $\epsilon$, and we report a single row per dataset.
A similar pattern emerges: compared to BPS$(1-\alpha)$ as reported in \cref{tab:partially-valid:full}, APS$(1-\alpha)$ again violates the desired conditional coverage guarantee.

For the case of credal sets with unknown validity, the situation differs. As discussed in \cref{sec:CaseIII}, our guarantee can be leveraged for any set predictor, including APS. In particular, calibrated APS can also attain the PAC-style conditional coverage guarantee. Nevertheless, BPS continues to explicitly exploit epistemic uncertainty through the credal set representation, whereas APS operates on a single aggregated distribution and ignores this structure. In \cref{sec:appendix:toy}, we illustrate this difference through toy examples. The full comparison of APS- and BPS-based predictors for this setup 
under different calibration strategies is given in \cref{tab:real-world:full}.

% \paragraph{One important remark.}
% We emphasize that although we refer to this method as APS in our experiments, it should not be confused with the original APS approach from the literature \citep{romano2020classification}. The constructions considered here do not appear in prior work; rather, they correspond to our calibration strategy applied to APS as a base set predictor. In contrast, the original APS method is represented in our comparisons by APS$(1-\alpha)$ when used without calibration and by APS with MargZero calibration when calibrated in the standard conformal manner.
\paragraph{Remark (APS-based set predictors).}
We emphasize that although we refer to these methods as APS-based in our 
experiments, they should not be confused with the original APS approach from 
the literature \citep{romano2020classification}. The constructions considered 
here do not appear in prior work; rather, they correspond to our calibration 
strategies applied to APS as a base set predictor. Only APS$(1-\alpha)$ 
(APS without calibration) and APS with MargZero calibration (the standard 
conformal calibration) correspond to the original APS method; the remaining 
combinations are new.

% \begin{table}[t!]
%     \centering
%     \caption{Performance of APS$(1-\alpha)$ on the partially valid credal sets for different datasets. Here $1-\alpha = 0.9.$}
%     \resizebox{\columnwidth}{!}{
%         \begin{tabular}{l|cccc|cccc|cccc}
%         \toprule
%           & \multicolumn{4}{c}{\bfseries CIFAR-10} & \multicolumn{4}{|c}{\bfseries ChaosNLI} & \multicolumn{4}{|c}{\bfseries QualityMRI} \\
%          \cmidrule(r){2-5}
%          \cmidrule(r){6-9}
%          \cmidrule(r){10-13}
%           $\epsilon$& Credal Cvg. & Cond. Sat. & Marg. Cvg. & Set Size & Credal Cvg. & Cond. Sat. & Marg. Cvg. & Set Size & Credal Cvg. & Cond. Sat. & Marg. Cvg. & Set Size \\
%         \midrule
%          0.10 & 0.90 $\pm$ 0.01 & 0.53 $\pm$ 0.00 & 0.90 $\pm$ 0.00 & 1.05 $\pm$ 0.00 & 0.89 $\pm$ 0.02 & 0.56 $\pm$ 0.03 & 0.89 $\pm$ 0.01 & 1.89 $\pm$ 0.02 & 0.90 $\pm$ 0.08 & 0.11 $\pm$ 0.04 & 0.52 $\pm$ 0.04 & 1.14 $\pm$ 0.03 \\
%         0.20 & 0.80 $\pm$ 0.01 & 0.53 $\pm$ 0.00 & 0.90 $\pm$ 0.00 & 1.05 $\pm$ 0.00 & 0.79 $\pm$ 0.03 & 0.56 $\pm$ 0.03 & 0.89 $\pm$ 0.01 & 1.89 $\pm$ 0.02 & 0.79 $\pm$ 0.08 & 0.11 $\pm$ 0.04 & 0.52 $\pm$ 0.04 & 1.14 $\pm$ 0.03 \\
%         0.30 & 0.70 $\pm$ 0.01 & 0.53 $\pm$ 0.00 & 0.90 $\pm$ 0.00 & 1.05 $\pm$ 0.00 & 0.70 $\pm$ 0.04 & 0.56 $\pm$ 0.03 & 0.89 $\pm$ 0.01 & 1.89 $\pm$ 0.02 & 0.71 $\pm$ 0.10 & 0.11 $\pm$ 0.04 & 0.52 $\pm$ 0.04 & 1.14 $\pm$ 0.03 \\
%         \bottomrule
%         \end{tabular}
%         }
%     \label{tab:partially-valid:full:APS}
% \end{table}
\begin{table}[t!]
    \centering
    \caption{Performance of APS$(1-\alpha)$ on the partially valid credal sets for different datasets. Since APS operates on the mean predictor, which does not depend
    on $\epsilon$, the results are identical for all values of $\epsilon$.
    Here $1-\alpha = 0.9.$}
    % \resizebox{\columnwidth}{!}{
        \begin{tabular}{l|ccc}
        \toprule
        Dataset & Cond. Sat. & Marg. Cvg. & Set Size \\
        \midrule
        {\bfseries CIFAR-10}   & 0.53 $\pm$ 0.00 & 0.90 $\pm$ 0.00 & 1.05 $\pm$ 0.00 \\
        {\bfseries ChaosNLI}   & 0.56 $\pm$ 0.03 & 0.89 $\pm$ 0.01 & 1.89 $\pm$ 0.02 \\
        {\bfseries QualityMRI} & 0.11 $\pm$ 0.04 & 0.52 $\pm$ 0.04 & 1.14 $\pm$ 0.03 \\
        \bottomrule
        \end{tabular}
        % }
    \label{tab:partially-valid:full:APS}
\end{table}

\begin{table*}[t!]
    \centering
		\caption{Performance comparison of different set predictors using different calibration approaches on credal sets with unknown validity across three real-world datasets. $1-\alpha = 0.9$ and $1-\beta = 0.9$.}
        \label{tab:real-world:full}
  \resizebox{\textwidth}{!}{
\begin{tabular}{ll|ccc|ccc|ccc}
\toprule
& &  \multicolumn{3}{c}{\textbf{CIFAR-10}} & \multicolumn{3}{|c}{\textbf{ChaosNLI}} & \multicolumn{3}{|c}{\textbf{QualityMRI}}\\
\cmidrule(r){3-5}
\cmidrule(r){6-8}
\cmidrule(r){9-11}
base& Calibration & Cond. Sat. & Marg. Cvg. & Set Size & Cond. Sat. & Marg. Cvg. & Set Size & Cond. Sat. & Marg. Cvg. & Set Size \\
\midrule
\multirow{5}{*}{\rotatebox{90}{BPS}}&BPS $(1-\alpha)$ & 0.52 $\pm$ 0.00 & 0.90 $\pm$ 0.00 & 1.09 $\pm$ 0.00 & 0.61 $\pm$ 0.01 & 0.89 $\pm$ 0.01 & 2.00 $\pm$ 0.02 & 0.33 $\pm$ 0.05 & 0.81 $\pm$ 0.02 & 1.36 $\pm$ 0.04 \\
&CondSatFirst & 0.90 $\pm$ 0.01 & 0.98 $\pm$ 0.00 & 1.45 $\pm$ 0.05 & 0.91 $\pm$ 0.02 & 0.95 $\pm$ 0.01 & 2.40 $\pm$ 0.04 & 0.94 $\pm$ 0.04 & 0.99 $\pm$ 0.01 & 1.93 $\pm$ 0.02 \\
&CondSatZero & 0.52 $\pm$ 0.00 & 0.90 $\pm$ 0.00 & 1.09 $\pm$ 0.00 & 0.90 $\pm$ 0.01 & 0.95 $\pm$ 0.01 & 2.35 $\pm$ 0.04 & 0.77 $\pm$ 0.10 & 0.96 $\pm$ 0.03 & 1.81 $\pm$ 0.07 \\
&MeanCondFirst & 0.75 $\pm$ 0.02 & 0.93 $\pm$ 0.00 & 1.17 $\pm$ 0.01 & 0.70 $\pm$ 0.03 & 0.90 $\pm$ 0.01 & 2.08 $\pm$ 0.03 & 0.83 $\pm$ 0.05 & 0.97 $\pm$ 0.01 & 1.85 $\pm$ 0.04 \\
&MargZero & 0.34 $\pm$ 0.17 & 0.90 $\pm$ 0.00 & 1.09 $\pm$ 0.01 & 0.71 $\pm$ 0.05 & 0.91 $\pm$ 0.02 & 2.09 $\pm$ 0.05 & 0.66 $\pm$ 0.09 & 0.93 $\pm$ 0.03 & 1.71 $\pm$ 0.07 \\
\midrule
\multirow{5}{*}{\rotatebox{90}{APS}}&APS $(1-\alpha)$ & 0.48 $\pm$ 0.00 & 0.89 $\pm$ 0.00 & 0.99 $\pm$ 0.00 & 0.44 $\pm$ 0.02 & 0.84 $\pm$ 0.01 & 1.77 $\pm$ 0.01 & 0.17 $\pm$ 0.06 & 0.76 $\pm$ 0.03 & 1.24 $\pm$ 0.04 \\
&CondSatFirst & 0.90 $\pm$ 0.01 & 0.99 $\pm$ 0.00 & 1.40 $\pm$ 0.07 & 0.91 $\pm$ 0.01 & 0.95 $\pm$ 0.01 & 2.39 $\pm$ 0.03 & 0.94 $\pm$ 0.03 & 0.99 $\pm$ 0.01 & 1.93 $\pm$ 0.02 \\
&CondSatZero & 0.48 $\pm$ 0.00 & 0.89 $\pm$ 0.00 & 0.99 $\pm$ 0.00 & 0.89 $\pm$ 0.02 & 0.94 $\pm$ 0.01 & 2.31 $\pm$ 0.04 & 0.73 $\pm$ 0.12 & 0.96 $\pm$ 0.03 & 1.80 $\pm$ 0.07 \\
&MeanCondFirst & 0.80 $\pm$ 0.01 & 0.93 $\pm$ 0.00 & 1.08 $\pm$ 0.01 & 0.74 $\pm$ 0.03 & 0.90 $\pm$ 0.01 & 2.03 $\pm$ 0.03 & 0.83 $\pm$ 0.05 & 0.97 $\pm$ 0.02 & 1.85 $\pm$ 0.03 \\
&MargZero & 0.56 $\pm$ 0.09 & 0.90 $\pm$ 0.00 & 1.01 $\pm$ 0.01 & 0.75 $\pm$ 0.03 & 0.91 $\pm$ 0.01 & 2.06 $\pm$ 0.05 & 0.63 $\pm$ 0.11 & 0.93 $\pm$ 0.03 & 1.70 $\pm$ 0.09 \\
\bottomrule
\end{tabular}

    }
\end{table*}

% $$\gL(\vx_i, \lambda) = 1 - \underbrace{C(x_i, \lambda) \cdot p_i}_\text{conditional cvg. at $x_i$} $$
% Therefore, CP will guarantee
% \begin{align*}
% \E_{\{\vx_i\}_{i = 1}^{n+1}}[1 - C(x_i, \lambda^\star) \cdot p_i] &\le \alpha \\
% \E_{\{\vx_i\}_{i = 1}^{n+1}}[C(x_i, \lambda^\star) \cdot p_i] &\geq 1 - \alpha \\
% \E_{\{\vx_i\}_{i = 1}^{n+1}}[\E_{y_i \sim p_i}C_{y_i}(x_i, \lambda^\star)] &\geq 1 - \alpha
% \end{align*}
% \paragraph{Comparison of two conformalization strategies.}
\paragraph{Remark (comparison of two conformalization strategies).}
Conformalized credal set predictors \citep{javanmardi2024conformalized} (as explained in \cref{sec:experiments:partially-valid}) take any probabilistic classifier together with a set of first-order calibration data and turn it into a credal set predictor that guarantees the true label distribution lies in the credal set with a user-specified probability, hence yielding a partially valid credal set predictor. In \cref{sec:CaseIII}, we also showed that, given the risk defined in (\ref{eq:risk:satisfaction}) and first-order calibration data, our method can instead transform a probabilistic classifier directly into a conformal set predictor with a PAC-style conditional coverage guarantee.

A comparison between these two approaches can be seen in our results: \cref{tab:partially-valid:full} (BPS$(1-\alpha)$ with $\epsilon = 0.1$) can be compared with \cref{tab:real-world:full} (APS with CondSatFirst calibration). It can be observed that the second approach, direct conformalization for conditional coverage, yields more efficient (smaller) prediction sets, while the first approach, first constructing conformalized credal sets and then applying optimal BPS with nominal $1-\alpha$, results in larger sets. This is mainly because prediction sets constructed from invalid credal sets may still satisfy conditional coverage, so enforcing credal validity first can be unnecessarily conservative.

\section{Experiments on Out-of-Distribution Data}\label{sec:appendix:ood}
A domain where epistemic uncertainty arises is when dealing with out-of-distribution (OOD) data for prediction, where data points come from a distribution that deviates from the original training distribution. In such cases, incorporating epistemic uncertainty into set prediction is even more important. To that end, we consider an experimental setup using CIFAR-10-C \citep{hendrycks2019robustness}, a corrupted version of CIFAR-10 that is often used to evaluate OOD robustness. Here, as the corruption type, we choose Gaussian noise with severity levels from 1 to 5.

Similar to the in-distribution setting, we apply calibration (if any) solely on in-distribution data. During testing, we use the credal set predictor to provide predictions for OOD input data. \cref{tab:ood:cifar10c} compares the performance of different set predictors under three severity levels: 1, 3, and 5.
As the severity increases, the level of corruption also increases, and thus it is expected that the credal sets become larger. This is reflected in the set size of the BPS-based approaches, as they are adaptive to epistemic uncertainty. Furthermore, as expected, BPS outperforms APS in the majority of cases, both with and without calibration, in terms of conditional coverage satisfaction, by incorporating epistemic uncertainty while producing larger sets.
\begin{table*}[t!]
    \centering
    \caption{Performance comparison of different set predictors using various calibration approaches on credal sets with unknown validity on CIFAR-10-C (Gaussian noise corruption across different severity levels). $1-\alpha = 0.9$ and $1-\beta = 0.9$.}
    \label{tab:ood:cifar10c}
    \resizebox{\textwidth}{!}{
\begin{tabular}{ll|ccc|ccc|ccc}
    \toprule
& & \multicolumn{3}{c}{\textbf{Severity 1}} & \multicolumn{3}{c}{\textbf{Severity 3}} & \multicolumn{3}{c}{\textbf{Severity 5}} \\
    \cmidrule(r){3-5} \cmidrule(r){6-8} \cmidrule(r){9-11}
base & Calibration & Cond. Sat. & Marg. Cvg. & Set Size & Cond. Sat. & Marg. Cvg. & Set Size & Cond. Sat. & Marg. Cvg. & Set Size \\
    \midrule
    \multirow{5}{*}{\rotatebox{90}{BPS}} & BPS $(1-\alpha)$ & 0.61 $\pm$ 0.00 & 0.87 $\pm$ 0.00 & 1.63 $\pm$ 0.01 & 0.43 $\pm$ 0.00 & 0.60 $\pm$ 0.00 & 2.08 $\pm$ 0.01 & 0.32 $\pm$ 0.00 & 0.48 $\pm$ 0.00 & 2.11 $\pm$ 0.01 \\
    & CondSatFirst & 0.90 $\pm$ 0.01 & 0.96 $\pm$ 0.00 & 2.59 $\pm$ 0.13 & 0.68 $\pm$ 0.02 & 0.77 $\pm$ 0.02 & 3.52 $\pm$ 0.21 & 0.57 $\pm$ 0.02 & 0.67 $\pm$ 0.02 & 3.57 $\pm$ 0.20 \\
    & CondSatZero & 0.61 $\pm$ 0.00 & 0.87 $\pm$ 0.00 & 1.62 $\pm$ 0.01 & 0.43 $\pm$ 0.00 & 0.60 $\pm$ 0.00 & 2.08 $\pm$ 0.01 & 0.32 $\pm$ 0.00 & 0.48 $\pm$ 0.00 & 2.10 $\pm$ 0.01 \\
    & MeanCondFirst & 0.78 $\pm$ 0.01 & 0.90 $\pm$ 0.00 & 1.81 $\pm$ 0.02 & 0.52 $\pm$ 0.01 & 0.65 $\pm$ 0.01 & 2.36 $\pm$ 0.03 & 0.40 $\pm$ 0.01 & 0.53 $\pm$ 0.01 & 2.39 $\pm$ 0.03 \\
    & MargZero & 0.51 $\pm$ 0.09 & 0.87 $\pm$ 0.00 & 1.62 $\pm$ 0.01 & 0.40 $\pm$ 0.03 & 0.60 $\pm$ 0.01 & 2.08 $\pm$ 0.02 & 0.31 $\pm$ 0.02 & 0.48 $\pm$ 0.00 & 2.10 $\pm$ 0.02 \\
    \midrule
    \multirow{5}{*}{\rotatebox{90}{APS}} & APS $(1-\alpha)$ & 0.54 $\pm$ 0.00 & 0.82 $\pm$ 0.00 & 1.26 $\pm$ 0.00 & 0.35 $\pm$ 0.00 & 0.49 $\pm$ 0.00 & 1.48 $\pm$ 0.01 & 0.25 $\pm$ 0.00 & 0.36 $\pm$ 0.00 & 1.45 $\pm$ 0.01 \\
    & CondSatFirst & 0.90 $\pm$ 0.01 & 0.96 $\pm$ 0.01 & 2.46 $\pm$ 0.21 & 0.68 $\pm$ 0.03 & 0.75 $\pm$ 0.03 & 3.32 $\pm$ 0.33 & 0.57 $\pm$ 0.04 & 0.64 $\pm$ 0.03 & 3.30 $\pm$ 0.32 \\
    & CondSatZero & 0.54 $\pm$ 0.00 & 0.82 $\pm$ 0.00 & 1.26 $\pm$ 0.00 & 0.35 $\pm$ 0.00 & 0.49 $\pm$ 0.00 & 1.48 $\pm$ 0.01 & 0.25 $\pm$ 0.00 & 0.36 $\pm$ 0.00 & 1.45 $\pm$ 0.01 \\
    & MeanCondFirst & 0.78 $\pm$ 0.01 & 0.88 $\pm$ 0.01 & 1.48 $\pm$ 0.02 & 0.47 $\pm$ 0.01 & 0.57 $\pm$ 0.01 & 1.81 $\pm$ 0.04 & 0.35 $\pm$ 0.01 & 0.43 $\pm$ 0.01 & 1.79 $\pm$ 0.04 \\
    & MargZero & 0.60 $\pm$ 0.06 & 0.84 $\pm$ 0.01 & 1.31 $\pm$ 0.02 & 0.39 $\pm$ 0.02 & 0.51 $\pm$ 0.01 & 1.56 $\pm$ 0.02 & 0.28 $\pm$ 0.01 & 0.38 $\pm$ 0.01 & 1.53 $\pm$ 0.02 \\
    \bottomrule
\end{tabular}
    }
\end{table*}
\end{document}